\documentclass[twoside]{article}

\usepackage{style}
\usepackage{graphicx}
\usepackage{color}

\usepackage{amsmath}
\usepackage{graphicx}

\usepackage{microtype}
\usepackage{graphicx}
\usepackage{booktabs} 
\usepackage{subfig}
\usepackage{subcaption}
\usepackage{algorithm}
\usepackage{algorithmicx}
\usepackage{algpseudocode}
\usepackage{multirow}
\usepackage{comment}
\usepackage{microtype}
\usepackage{hyperref}
\usepackage{amsmath}
\usepackage{amssymb}
\usepackage{mathtools}
\usepackage{amsthm}
\usepackage{bbm}
\usepackage{bm}
\usepackage{xcolor}
\usepackage{mathtools}
\usepackage{caption}
\usepackage{mathrsfs}

\theoremstyle{plain}
\newtheorem{theorem}{Theorem}[section]

\theoremstyle{definition}
\newtheorem{definition}[theorem]{Definition}

\theoremstyle{remark}

\DeclareMathOperator{\Tr}{Tr}

\DeclarePairedDelimiter{\norm}{\lVert}{\rVert}

\newcommand{\mtx}{\bm} 
\newcommand{\vct}{\bm} 

\newcommand{\E}{\mathbbm{E}}
\newcommand{\R}{\mathbbm{R}}

\def\H{\mathcal{H}}

\usepackage[round]{natbib}
\renewcommand{\bibname}{References}
\renewcommand{\bibsection}{\subsubsection*{\bibname}}

\date{}

\title{Inference and Interference: The Role of Clipping, Pruning and Loss Landscapes in Differentially Private Stochastic Gradient Descent}

 \author{Lauren Watson,\textsuperscript{1}
Eric Gan,\textsuperscript{2}
Mohan Dantam ,\textsuperscript{1}
Baharan Mirzasoleiman,\textsuperscript{2}
Rik Sarkar \textsuperscript{1}\\
\textsuperscript{1}{School of Informatics, University of Edinburgh}\\
\textsuperscript{2}{University of California, Los Angeles}\\
lauren.watson@ed.ac.uk, baharan@cs.ucla.edu, rsarkar@inf.ed.ac.uk}
\begin{document}

\maketitle

\begin{abstract}
Differentially private stochastic gradient descent (DP-SGD) is known to have poorer training and test performance on large neural networks, compared to ordinary stochastic gradient descent (SGD). In this paper, we perform a detailed study and comparison of the two processes and unveil several new insights. By comparing the behavior of the two processes separately in early and late epochs, we find that while DP-SGD makes slower progress in early stages, it is the behavior in the later stages that determines the end result. This separate analysis of the clipping and noise addition steps of DP-SGD shows that while noise introduces errors to the process, gradient descent can recover from these errors when it is not clipped, and clipping appears to have a larger impact than noise. These effects are amplified in higher dimensions (large neural networks), where the loss basin occupies a lower dimensional space. We argue theoretically and using extensive experiments that magnitude pruning can be a suitable dimension reduction technique in this regard, and find that heavy pruning can improve the test accuracy of DPSGD.

\end{abstract}

\section{INTRODUCTION}
\label{sec:intro}

The wide-ranging potential applications of deep learning models trained on large datasets has recently raised concerns about the privacy properties of these models. The current method to enforce privacy while training such models is differential privacy~\citep{Dwork06differentialprivacy}.  The most popular of these methods is Differentially Private Stochastic Gradient Descent (DP-SGD)~\citep{abadi2016deep}. By restricting the maximum gradient norm and adding suitable noise at every step of ordinary Stochastic Gradient Descent (SGD), it ensures the rigorous privacy properties of differential privacy and defends against privacy attacks~\citep{shokri2017membership,fredrikson2015model,nasr2019comprehensive,carlini2019secret}.

The perturbations produced by noise addition and gradient clipping come at the cost of reduced utility, and the practical performance of DP-SGD -- as measured by common loss and accuracy measures -- is found to be consistently inferior to that of SGD, even on standard benchmark tasks~\citep{papernot2021tempered}. This gap is governed by several entities including the class of models, the gradients generated by SGD, the perturbation induced by DP-SGD, and the {\em loss landscape} over the space of models $\H$ where the training process takes place. The relations between these entities are complex, and are the subject of this paper. It has been conjectured that the primary weakness of DP-SGD is in the early rounds, where the noise may inhibit it from identifying a good basin~\citep{ganesh2023public} within the loss landscape. Having missed the opportunity, it is constrained to the quality of the basin it finds. On the other hand, studies on the SGD process have found that it is often possible to arrive at equally good minima from different starting points of SGD~\citep{frankle2020linear} and for some common learning tasks, these different loss basin floors are connected~\citep{garipov2018loss}. The performance of DP-SGD is known to deteriorate with the dimension of $\H$, or equivalently, with the number of parameters of the model, since to privatize a greater number of parameters requires greater noise. Thus, several methods have been proposed to reduce the number of parameters by pruning, gradient sparsification and others~\citep{luo2021scalable,li2022does,zhou2021bypassing}.

\subsection{Contributions}

We provide several insights into the behavior of DP-SGD by comparing it with the fine grained behavior of ordinary SGD. The results include observations about the nature of the loss basins, why DP-SGD fails to find a floor (minimum) of the basin as effectively as SGD does, how clipping and noise addition contribute to the process, and how pruning improves the performance. Below, we summarize these findings.

If the early rounds of training are decisive, as suggested in~\citep{ganesh2023public} then the process -- SGD or DP-SGD -- used in the first few (say $k$ epochs) will determine the quality of the model, irrespective of the process operating in later epochs. To test this idea, we divided the $T$ step training process into Phase 1 (first $k$ epochs) and Phase 2 (remaining $T-k$ epochs). Then four models were trained corresponding to using SGD or DP-SGD in each of the phases. The results consistently show that the process used in Phase 2 has a larger impact. In Figure~\ref{fig:sgd-dpsgd-combos}, the loss and accuracy levels at the end of Phase 1 are clearly different for SGD and DP-SGD. But in Phase 2, the results change, so that the SGD processes in Phase 2 converge to similar low losses, while the processes running DP-SGD in Phase 2 converge to higher loss. Thus, contrary to the conjecture in~\citet{ganesh2023public}, it is the Phase 2 process that is decisive, suggesting that the  perturbations prevent DP-SGD from finding a good solution irrespective of initial steps. Similar results on other models are found in Section~\ref{sec:experiments}. 

\begin{figure}[htbp]
    \centering
    \includegraphics[width=1.5in]{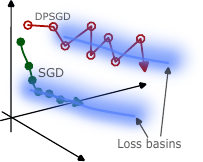}
    \caption{Loss Basins with $1$-D Basin floors (shown as blue lines). Loss basin dimension is low compared to ambient dimension. SGD finds and stays at the floor. DP-SGD finds it hard to stay at the floor as the large magnitude of random noise easily throws it far from the floor, and loss increases rapidly in random directions.}
    \label{fig:basins-schematic}
\end{figure}

Next, we turn to studying the loss landscape and loss basins that guide the training process. In experiments (also seen in~\citep{frankle2020linear,garipov2018loss}), SGD usually finds a model at the floor (or plateau) of a loss basin. At the floor, there are a few  directions where the loss stays flat. We find that in most other directions, the loss rises rapidly. The conclusion from this result is that while the ambient dimension (the number of parameters or weights) of a neural network is large, a loss basin lies only in a relatively lower dimensional subspace. The idea is shown schematically in Figure~\ref{fig:basins-schematic}, where one dimensional loss basins are embedded in $3D$ space. In almost any direction from a basin floor, the loss rises rapidly. When the number of orthogonal dimensions are larger, the effect is more pronounced, and has major consequences for DP-SGD: since the noise induced by DP-SGD is a random vector drawn from a high dimensional Gaussian, and therefore of large magnitude. The result is that DP-SGD cannot find the basin floor or remain close to it.  (Figure~\ref{fig:basins-schematic}).

Further, even with quite short lengths of Phase 1, SGD can find its basin, after which SGD processes end on the same basin floor. However, when DP-SGD is used in Phase 2, processes usually do not end in the same basin. In fact, in many cases, DP-SGD jumps from one basin to another. Even when models do converge to the same basin, DP-SGD does not find a floor that SGD does consistently.


We then examine the effects of clipping and noise addition separately. In experiments, while SGD can tolerate noise addition of the same scale as that required to guarantee $\epsilon=1$ differential privacy, clipping to the same level as an $\epsilon=1$ DP-SGD model significantly reduces SGD test accuracy.  We use analysis at both small scales (a few training steps) and larger scales (a full training run) to provide evidence that the clipping operation prevents the DP-SGD model from recovering from noise addition that would otherwise have little impact on the training ability for the model.

Given this understanding of the low dimensional loss basins, and the impact of clipping and noise addition on finding these solutions, we show that suitable pruning can mitigate these effects. As with previous work (\citep{luo2021scalable}) we find that appropriate pruning improves the accuracy of DP-SGD and that pruned DP-SGD models better replicate the low dimensional behavior of SGD models. Even when the pruning is carried out using only $5\%$ of the data or with a different dataset, DP-SGD on the pruned network generalizes better than DP-SGD on the full network. Separate experiments on clipping and noise addition show that pruning significantly reduces the impact of the clipping operation.

Theoretically, while it is clear from previous work that the effect of noise increases with the number of parameters~\citep{bassily2014private} and therefore pruning will reduce the impact of noise, the effect of clipping and the corresponding potential benefits of pruning are less immediate. We show that  pruning can in fact reduce the harmful effects of clipping, by defining a term $R$ based on the variance along each dimension and the true gradient norm, which decreases with improvement in clipped gradient. The reduction of this term with pruning is verified empirically.

\section{TECHNICAL BACKGROUND}

\subsection{Differential privacy and Stochastic Gradient Descent}

Differential privacy as outlined in Definition~\ref{def:DP} guarantees that the presence or absence of a single datapoint in the underlying training dataset will have a limited impact on the output model, with lower values of $\epsilon$ corresponding to stronger privacy guarantees.
\begin{definition}[\bf Neighboring Databases]
Two databases $D, D'$ are neighboring if $H(D, D') \leq 1$, where $H(\cdot, \cdot)$ represents the hamming distance.
\end{definition}
\begin{definition}[\bf Differential Privacy~\citep{Dwork06differentialprivacy}] \label{def:DP}
  A randomized algorithm $M$ satisfies $(\epsilon,\delta)$-differential privacy if for all neighboring databases $D$ and $D'$ and for all possible outputs $O\subseteq \text{Range}(M)$, \mbox{
  $
      \Pr[M(D) \in O] \leq e^{\epsilon} \cdot \Pr[M(D') \in O] +\delta
  $}
\end{definition}

Differentially private algorithms often rely on carefully calibrated noise addition to provably ensure differential privacy~\citep{dwork2006calibrating}. \citet{abadi2016deep} introduced a differentially private variant of the stochastic gradient descent algorithm (DP-SGD). Given data distribution $\mathcal{Z}=\mathcal{X}\times \mathcal{Y} $, loss function $L:(\mathbf{\theta}, \mathbf{z})\rightarrow \mathbb{R}$ for any $\mathbf{z} \in \mathcal{Z}$, learning rate $\eta\in\mathbb{R}^+$, and model parameters $\mathbf{\theta}_t$ at training step $t$, SGD calculates the gradient $\mathbf{g}_t=\nabla_{\mathbf{\theta}_t}L(\theta_t, \mathbf{z})$ and updates the model parameters to $\mathbf{\theta}_{t+1} = \mathbf{\theta}_{t}-\eta \mathbf{g}_t $. In contrast, DP-SGD performs each update step with a perturbed gradient, $\tilde{\mathbf{g}}_t$. The perturbation of the true gradient follows two distinct steps. Firstly, the gradient norm of $\mathbf{g}_t$ is clipped to have maximum norm $C$, with $\hat{\mathbf{g}}_t=\frac{\mathbf{g}_t}{ \text{max}(1, \frac{\lVert \mathbf{g}_t \rVert_2}{C})}$. Then Gaussian noise with scale $\sigma$ is added where $\sigma$ is determined by the required privacy guarantee $\epsilon$, with $\tilde{\mathbf{g}}_t=\hat{\mathbf{g}}_t+ \mathcal{N}(0, C\sigma^2\mathbb{I})$. Finally, the update is performed: $\theta_{t+1} = \theta_{t}-\eta \tilde{\mathbf{g}}_t$.

\subsection{Linear mode connectivity} \label{sec:linear-mode-connectivity}

Loss basins that contain different models (say, $\theta_0$ and $\theta_1$) can be compared using a technique called Linear Mode Connectivity~\citep{frankle2020linear}. In this technique, the loss values are computed for convex combinations of the two models: $(1-\alpha)*\theta_0 + \alpha\theta_1$ for $\alpha\in[0,1]$, that lie along the linear segment joining the two models. In practice losses are computed at discrete points, usually $30$~\citep{frankle2020linear}. Figure~\ref{fig:sgd-linear-mode} shows the typical such linear mode loss profile for two SGD models. The differences in initialization and inherent randomness in SGD data processing can cause the models to end in different basins with higher loss in between, as seen. It has been shown by~\citep{garipov2018loss} that the loss basins may be connected via non-linear paths. The linear mode connectivity technique is best suited to find models that are in the same linearly connected (convex) parts of a basin -- where the loss profile at all interpolated points are below some threshold above the loss at the end points, that is, there is no pronounced peak of loss in between the models. The height of this peak loss relative to the loss at the end points is called the {\em instability} of the training process. More formally, suppose $\mathcal{E}(\theta)$ is the training or test loss for model $\theta$ and $\mathcal{E}_{\alpha}(\theta_0, \theta_1)=\mathcal{E}((1-\alpha)\theta_0+\alpha\theta_1 )$ for $\alpha \in [0,1]$ is the loss for a convex combination of the two models .

\begin{definition}[Linear Interpolation Instability] Let
$\mathcal{E}_{sup}(\theta_0, \theta_1)=\sup_{\alpha }\mathcal{E}_{\alpha}(\theta_0, \theta_1)$ be the highest loss attained while interpolating between models $\theta_0$ and $\theta_1$ with $\alpha\in[0,1]$. Let $\bar{\mathcal{E}}_{\alpha}(\theta_0, \theta_1) = \text{Mean}(\mathcal{E}(\theta_0), \mathcal{E}(\theta_1))$ represent the mean loss of these models. Then the \textit{linear interpolation instability} is given by $\mathcal{E}_{sup}(\theta_0, \theta_1)-\bar{\mathcal{E}}_{\alpha}(\theta_0, \theta_1) $.

\end{definition}

High values of instability caused by a peak as in Figure~\ref{fig:sgd-linear-mode}, imply that the models do not both lie on the floor of the same local basin.

\section{THEORETICAL ANALYSIS}
\label{sec:theory}

In this section, we quantify the effect of clipping in gradient descent and show an example setting where pruning helps reduce this effect.
\subsection{Measuring the effect of clipping}
One possible way to analyse the change in the gradient step caused due to clipping is to measure the change in direction compared to the true gradient. Intuitively, if the direction doesn't change too much, it can be thought of as taking a smaller step compared to the true gradient in a similar direction. To this end, we observe that the projection of the clipped gradient onto the true gradient can be lower bounded by a quantity which depends on the variance of each dimension over data points and gradient norm. We denote it by $R$ which is defined below.
\begin{definition}\label{def:R}
Let $\mathcal{G} = \{\vct{g}^{(i)}\}_{i=1}^n$ be per example gradients, with mean $\bar{\vct{g}} = \frac{1}{n} \sum \vct{g}^{(i)}$. Then define
\begin{align}
R = \frac{\Tr(Cov(\mathcal{G}))}{\norm*{\bar{\vct{g}}}^2} = \frac{\frac{1}{n} \sum \norm*{\vct{g}^{(i)} - \bar{\vct{g}}}^2}{\norm*{\bar{\vct{g}}}^2}
\end{align}
\end{definition}
Consider the gradients $\vct{g}^{(i)}$ from \ref{def:R}. The clipped gradient $\vct{c}$ with clipping norm $C$ is the mean of per example clipped gradients i.e.,
\begin{align}
\vct{\Tilde{g}}^{(i)} &= \vct{g}^{(i)} \cdot \min\left(1,\frac{C}{\norm*{\vct{g}^{(i)}}}\right) \\
\vct{c} &= \frac{1}{n}\sum \vct{\Tilde{g}}^{(i)} \label{eq:true_clipped_grad}
\end{align}
However, to simplify the analysis and based on the empirical observation that most per-example gradients are larger than the clipping norm throughout the training in the standard noise settings of DP-SGD, we make the assumption that $C \leqslant \norm*{\vct{g}^{(i)}}$ for every $i$. This simplifies $\vct{c}$ to
\begin{equation}\label{eq:approx_clipped_grad}
    \vct{c} = \frac{C}{n} \sum \frac{\vct{g}^{(i)}}{\norm*{\vct{g}^{(i)}}}
\end{equation}
\begin{theorem} \label{thm:clipped_grad_alignment}
Assume that $\bar{\vct{g}} \neq 0$ and $\vct{g}^{(i)} \neq 0$ for all $i = 1, \dots n$. Let the clipped gradient with clipping norm $C$ be as above and $R$ from \ref{def:R}.

Then
\begin{align}
    \vct{c} \cdot \frac{\bar{\vct{g}}}{\norm*{\bar{\vct{g}}}} \geq C\left(1 - \frac{R}{2}\right)
\end{align}
\end{theorem}
From \ref{thm:clipped_grad_alignment} we can observe that the smaller the ratio $R$, the more aligned the accumulated clipped gradient is with the actual direction.

\subsection{Pruning to improve DPSGD}\label{sec:theory_example}
We construct a theoretical example in which performing pruning clips non-important weights, so that the true target function can still be learned by the pruned model. Later, we confirm empirically that in this scenario pruning reduces $R$, thereby minimizing the effect of clipping from DP-SGD.

Formally, we consider a binary classification dataset $\mathcal{D}$, where each example $\vct{x}^{(i)}$ is a vector in $\R^d, d = d_s + d_n$, which consists of a signal part and a noise part. The signal part contains a key vector $\vct{v}$ which determines the label plus Gaussian noise. The noise part consists of only Gaussian noise. The labels $y^{(i)}$ are drawn from $\{-1, 1\}$.
\begin{align}
\vct{x}^{(i)} &= (\vct{x}_s^{(i)}, \vct{x}_n^{(i)}) \\
\vct{x}_s^{(i)} &\in \R^{d_s}, \vct{x}_n^{(i)} \in \R^{d_n}   \\
\vct{x}_s^{(i)} &\sim \mathcal{N}(y_i \vct{v}, \sigma^2 \mtx{I}_{d_s}) \\ \vct{x}_n^{(i)} &\sim \mathcal{N}(0, \sigma^2 \mtx{I}_{d_n})
\end{align}
For simplicity we will consider training on the entire population, but similar results can be obtained for training on a sufficiently large sample using standard concentration inequalities.

Consider training a two layer linear model with hidden layer of size $m$
\begin{gather}
    f(\vct{x}) = \mtx{W}_2 \mtx{W}_1 \vct{x}, \\
    \mtx{W}_2 \in \R^{1 \times m}, \mtx{W}_1 \in \R^{m \times d}
\end{gather}
to minimize the mean squared error
\begin{align}
    \mathcal{L} = \E_{(\vct{x}, y) \sim \mathcal{D}} \left[ \frac{1}{2}\|f(\vct{x}^{(i)}) - \vct{y}^{(i)}\|^2 \right]
\end{align}
To aid in our analysis, we decompose
\begin{gather}
    \mtx{W}_1 = \begin{pmatrix} \mtx{W}_s & \mtx{W}_n \end{pmatrix} \\
    \mtx{W}_s \in \R^{m \times d_s}, \mtx{W}_n \in \R^{m \times d_n}
\end{gather}
Intuitively, $\mtx{W}_s$ consists of connections to the signal part of the input and $\mtx{W}_n$ consists of connections to the noise part of the input which correspond to non-important weights with regards to pruning.

We will consider clipping the first layer, as clipping the second layer is similar to reducing the dimension of the hidden layer $m$.

Under this setting, we obtain the following
\begin{theorem} \label{thm:magnitude_pruning}
Suppose that at initialization the following holds: $\mtx{W}_2^{\top} \mtx{W}_2 = \mtx{W}_1 \mtx{W}_1^{\top}$. Further assume that we train using gradient flow and converge to a global minimum. Then at this point
\begin{align}
    \mtx{W}_n = \mtx{0}
\end{align}
\end{theorem}
The initialization assumption has been used in previous theoretical works (\citep{arora2018optimization}). Moreover, the convergence to a global minima is not restrictive as it has been shown that for linear networks all local minima are global minima (\citep{kawaguchi2016deep}).

In  general, this Theorem \ref{thm:magnitude_pruning} shows that the weights in $\mtx{W}_n$ are expected to be small, hence will be clipped by magnitude pruning.

\section{EXPERIMENTS}
\label{sec:experiments}

In this section we provide empirical evidence for the hypothesized behavior and weaknesses of DP-SGD as discussed above. We also empirically demonstrate the benefits of pruning.

\textbf{Datasets and Models. } All experiments in this section use the publicly available CIFAR10 and CIFAR100 datasets~\citep{krizhevsky2009learning} using a Resnet18 model architecture~\citep{he2016deep}. For DP-SGD, the Opacus library~\citep{yousefpour2021opacus} was used to replace batch normalization layers with group normalization layers, as is required to maintain privacy guarantees. Experiments were also performed on the LeNet (CNN)~\citep{lecun1998gradient} architecture to illustrate that behaviors persist across drastically different model sizes ($\approx 11$ million parameters for Resnet18 vs $\approx 50,000$ for LeNet). Please see the Appendix for further experiments.

\textbf{Training. } Models trained via SGD used batch size 128 and learning rate 0.1 without momentum or data augmentation. Models trained with DP-SGD used batch size 1024 and learning rate 0.5.  All models were trained using cross entroy loss on NVIDIA RT2080Ti 11GB GPUs. Reported results are averaged over 3-5 training runs of 30 epochs.  All DP-SGD guarantees are for $\delta=1\times 10^{-5}$. See the Appendix for detailed descriptions of hyperparameter choices.

\textbf{Pruning. } Magnitude pruning was used to prune each network. Given the initialization point of the model to be pruned, this model was then trained for 20 epochs using the pre-training datasets outlined in Figure~\ref{fig:test-accs-pruning}. For pruning level $p\in[0,1]$, the smallest $(1-p)$ proportion of weights by magnitude were pruned in each layer. The remaining weights were reset to their initialization values.

\subsection{Properties of DP-SGD, SGD and Loss basins in high dimensional spaces}

\subsubsection{Importance of early Vs late epochs}

To test the conjecture in~\citep{ganesh2023public} that the poor performance of DP-SGD arises in the early stages of optimization, we divide the training/optimization process into Phase 1 ($k$ epochs) and Phase 2 ($T-k$ epochs). Two models are trained using SGD and DP-SGD for Phase 1, and then separate copies of each are trained with SGD and DP-SGD for Phase 2.

Figure~\ref{fig:sgd-dpsgd-combos}(a) demonstrates the loss results for the CNN model. Clearly, the losses are determined by the mode of operation in Phase 2. The two Phase 2 sequences that use SGD converge rapidly to similar losses, while the SGD-DPSGD sequence has its loss increase to appear at a similar level as the DPSGD-DPSGD sequence. Corresponding test accuracy are shown in Figure~\ref{fig:sgd-dpsgd-combos}(b), and results for Resnet18 (Figure~\ref{fig:sgd-dpsgd-combos} c) and d)) show analogous patterns. These results show that contrary to the suggestion in~\citep{ganesh2023public}, the early phases are not entirely decisive, and the inability to find a good solution even in a good basin is a major weak point of DP-SGD.

\begin{figure}[t]
\centering
\begin{tabular}{cc}
\includegraphics[width=0.45\linewidth]{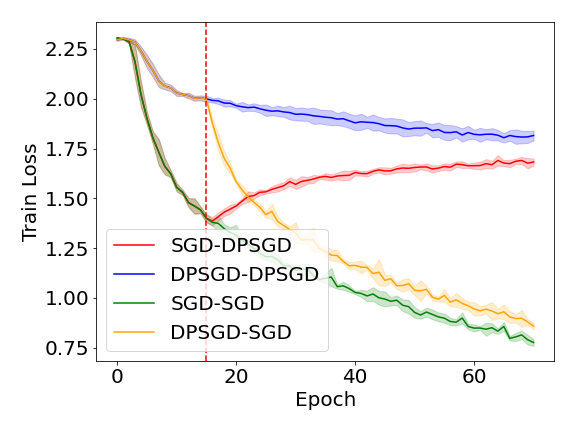} &
\includegraphics[width=0.45\linewidth]{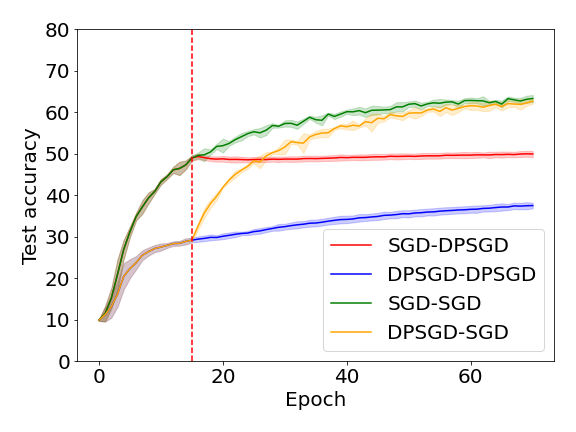}
\\
a) Train. Loss(CNN)  &b) Test Acc.(CNN)  \\
\includegraphics[width=0.45\linewidth]{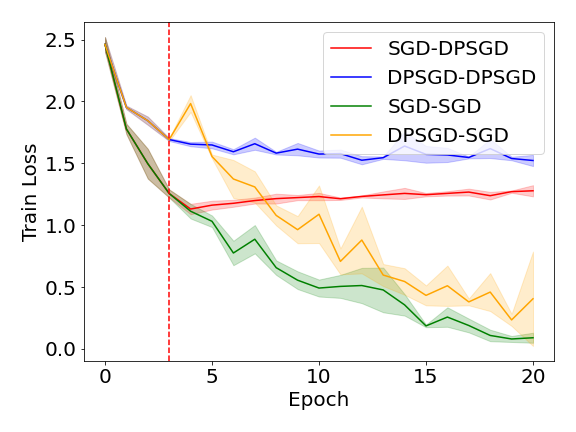} &
\includegraphics[width=0.45\linewidth]{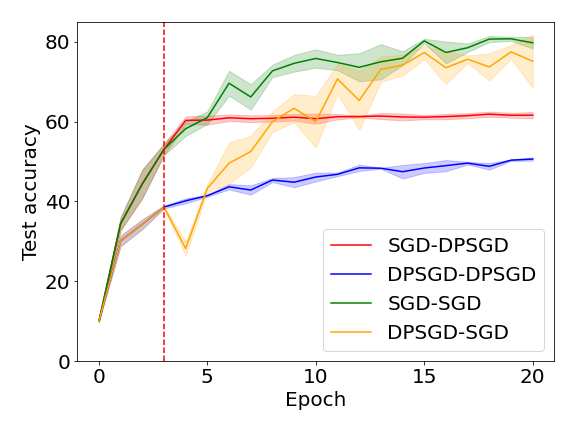}
\\
c) Train. Loss(Resnet18) &d) Test Acc.(Resnet18) \\
\end{tabular}
 \caption{Performance with different Phase 1 and Phase 2 training methods. Later training epochs determine the final performance of both models for CIFAR10. DP-SGD used noise multiplier $\sigma=0.55$ and maximum gradient norm $C=1.0$ in each epoch, with $\epsilon\approx 7$ after training. $k=15$ for CNN and $k=3$ for Resnet18.}
\label{fig:sgd-dpsgd-combos}
\end{figure}

\subsubsection{Loss basin dimensions}

By carrying out multiple runs of SGD we identify many basin locations with equal levels of loss at the basin floor. Figure~\ref{fig:sgd-linear-mode} shows that the linear mode interpolations between different SGD solutions contain a large barrier in between, suggesting that these are different loss basins, or parts of basins not linearly connected. The basin floors are at equally low loss levels, implying that there are many good solutions and initialization is not decisive.

To further evaluate the shape of the basins, we carried out experiments where two different models ($\theta_0, \theta_1$) were trained in Phase 2 starting from the same Phase 1 model with $k=5$. The linear mode interpolation profile between $\theta_0$ and $\theta_1$ shows an almost flat line, showing that the basin floor is almost flat (Figure~\ref{fig:random-vector-basin}). On the other hand, the interpolation in a random direction starting from $\theta_0$ shows the loss to rise rapidly. This was observed consistently on 100 random directions, where the maximum cosine similarity between the directions was $0.0012$, implying that the directions are mainly orthogonal, as expected when sampling random vectors in high dimensions.

This observation implies that the loss basin floor containing $\theta_0$ and $\theta_1$ is low dimensional compared to the ambient dimension (number of parameters). Any random vector -- such as the noise vector generated by DP-SGD --  from a model on the basin floor will take the process along a path of steep rise in loss. Note that the dimension of the loss basin itself is not critical to this observation, as long as a number of directions exist where loss increases rapidly.

\begin{figure}
    \centering
    \includegraphics[width=0.45\linewidth]{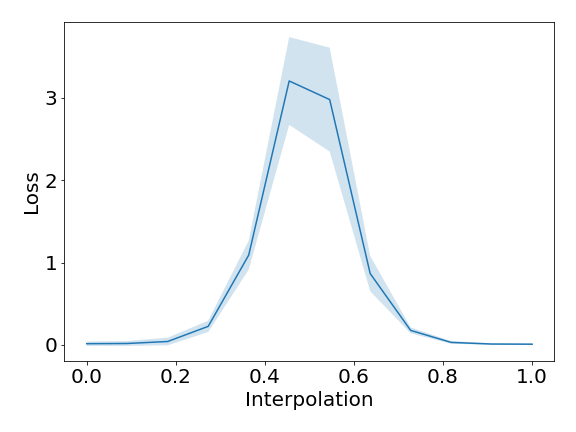}
    \caption{SGD models with different initialization have similar losses ($\theta_0$ at interpolation=0 and $\theta_1$ at interpolation=1) but are not linear mode connected, implying that they are in  different basins, or different regions of a non-convex basin.}
    \label{fig:sgd-linear-mode}
\end{figure}

\begin{figure}[h]
\centering
\includegraphics[width=0.45\linewidth]{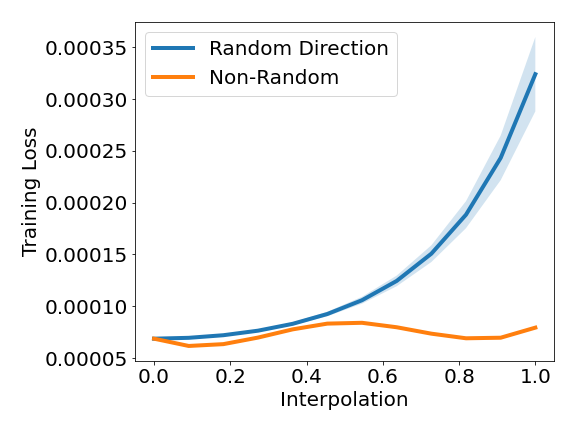}
 \caption{Comparison of the linear mode interpolation profiles of models $\theta_0, \theta_1$ in the same linearly connected basin at distance of 15 units, with a model obtained by moving from $\theta_0$ 15 units along a random vector in the ambient space. Evaluated with $100$ random vectors, with maximum cosine similarity of 0.0012. In random directions, loss increases rapidly.}
\label{fig:random-vector-basin}
\end{figure}

\subsubsection{Comparison of SGD and DP-SGD solutions with linear mode interpolations}

The linear interpolation instability for pairs of SGD models trained from the same Phase 1 model (for varying $k$) is shown in Figure~\ref{fig:linear-mode-sgd-dpsgd-comparison}(a). Clearly, SGD finds its basin early. However, DP-SGD shows high values of instability, implying that two runs of DP-SGD almost never end in the same linear basin, except perhaps for high values of both $k$ and $\epsilon$. When the two models are trained by SGD and DP-SGD respectively in Phase 2, the differences in loss profile can be seen in Figure~\ref{fig:linear-mode-sgd-dpsgd-interpolation} -- suggesting that the two models do not end in the same basin unless $\epsilon$ and $k$ are large. The corresponding results where Phase 1 uses DP-SGD are shown in Figure\ref{fig:linear-mode-sgd-dpsgd-interpolation-dpsgd-first} In this case, since the starting point is at a high loss and far from the basin floor, the Phase 2 DP-SGD always ends in a different basin.

In fact, as shown in Figure~\ref{fig:comp-to-pretrained}, If we interpolate the Phase 1 model with Phase 2 model, we see that in all cases, DP-SGD does not stay in the same basin as in Phase 1 ($k=10$), and always ends at a different basin.

\begin{figure}[h]
\centering
\begin{tabular}{cc}
\includegraphics[width=0.45\linewidth]{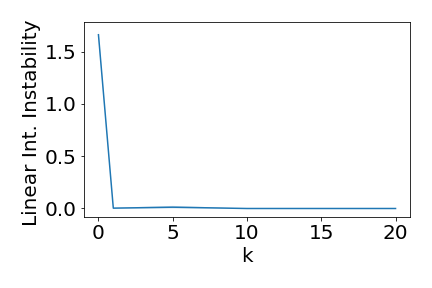} &
\includegraphics[width=0.45\linewidth]{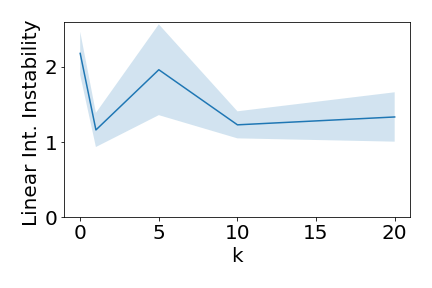}
\\
a) SGD  &b) DP-SGD ($\epsilon$=1)  \\
\includegraphics[width=0.45\linewidth]{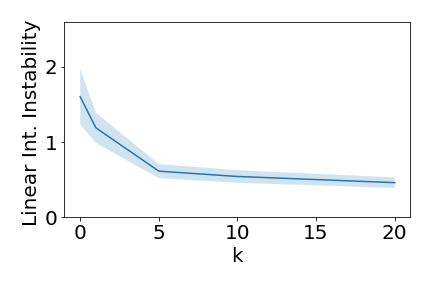} &
\includegraphics[width=0.45\linewidth]{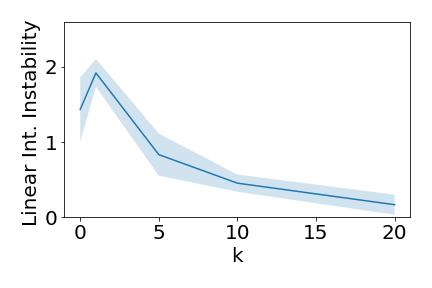}
\\
c) DP-SGD ($\epsilon$=5)  &d) DP-SGD ($\epsilon$=50)  \\
\end{tabular}
 \caption{SGD and DP-SGD Instability. Linear interpolation stability is shown using 30 values of $\alpha$ between two models trained with different copies after Phase 1. Higher y-values represent less linear mode stability.}
\label{fig:linear-mode-sgd-dpsgd-comparison}
\end{figure}

\begin{figure}[h]
\centering
\begin{tabular}{cc}
\includegraphics[width=0.45\linewidth]{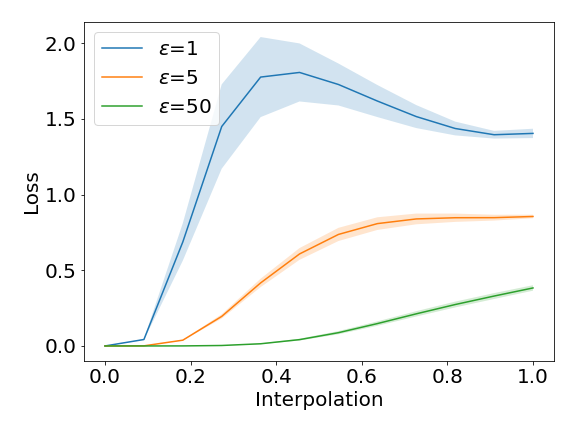} &
\includegraphics[width=0.45\linewidth]{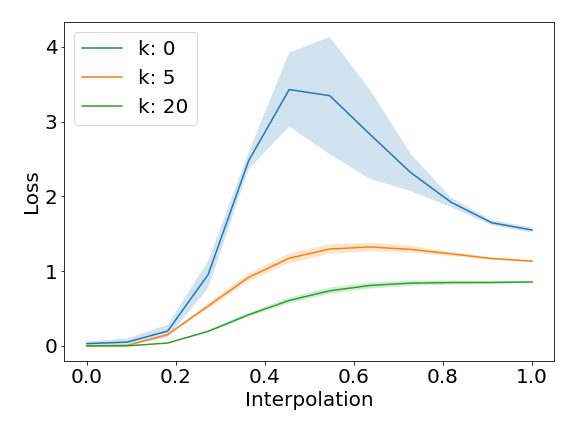}
\\
a) k=20  &b) $\epsilon=5$ \\
\end{tabular}
 \caption{Linear model interpolation between models $\theta_0$ trained for 20 epochs with SGD in Phase 2 and $\theta_1$ trained for 20 epochs with DP-SGD in Phase 2. Phase 1 was trained with SGD.}
\label{fig:linear-mode-sgd-dpsgd-interpolation}
\end{figure}

\begin{figure}[h]
\centering
\begin{tabular}{cc}
\includegraphics[width=0.45\linewidth]{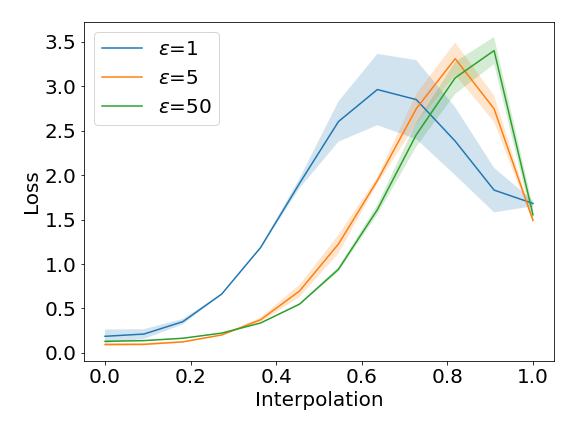} &
\includegraphics[width=0.45\linewidth]{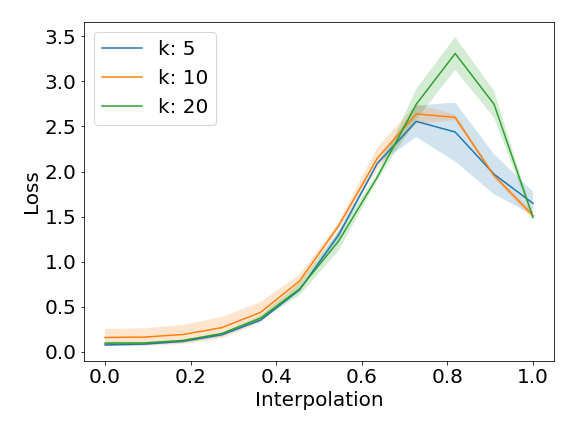}
\\
a) k=20  &b) $\epsilon=5$ \\
\end{tabular}
 \caption{Linear mode interpolation. $\theta_0$ trained for 20 epochs with SGD in Phase 2, and $\theta_1$ trained for 20 epochs with DP-SGD in Phase 2. Phase 1 trained with DP-SGD.}
\label{fig:linear-mode-sgd-dpsgd-interpolation-dpsgd-first}
\end{figure}

\begin{figure}[h]
\centering
\begin{tabular}{cc}
\includegraphics[width=0.45\linewidth]{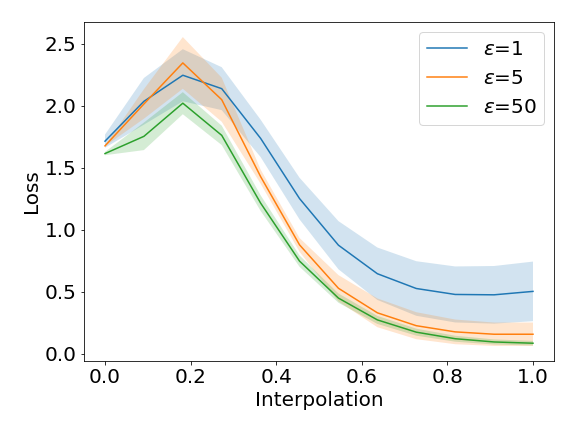} &
\includegraphics[width=0.45\linewidth]{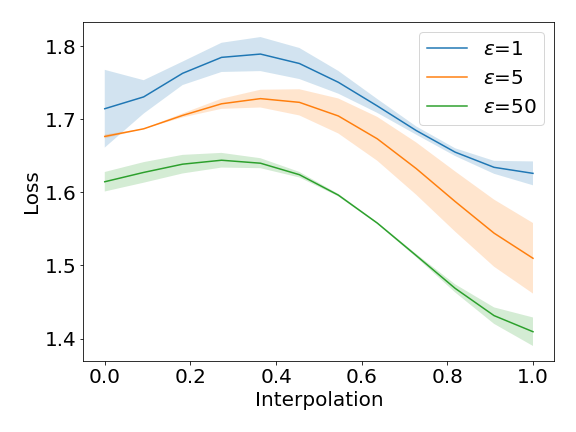}
\\
a) P1: DPSGD; P2: SGD  &b) P1 \& 2: DP-SGD \\
\includegraphics[width=0.45\linewidth]{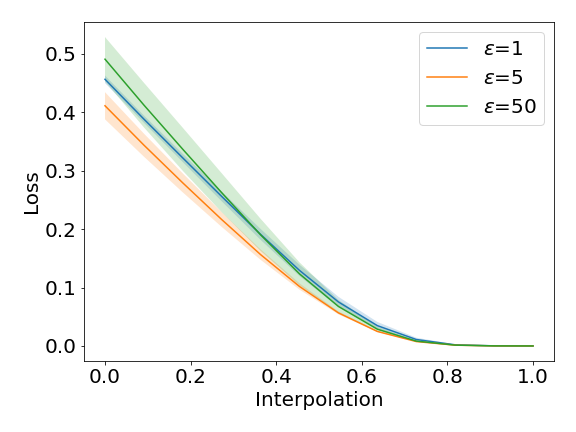} &
\includegraphics[width=0.45\linewidth]{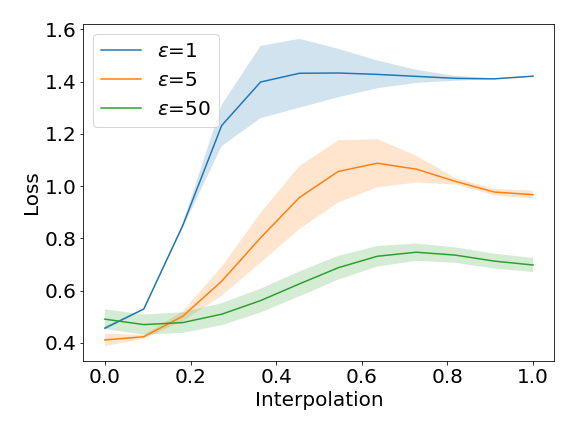}
\\
c) P1 \& 2: SGD   &d) P1: SGD; P2: DPSGD \\
\end{tabular}
 \caption{Linear mode interpolation. $\theta_0$ is model at end of Phase 1; $\theta_1$ is model at end of phase 2. SGD Phase 1 followed by also SGD Phase 2 is the only setting where the model clearly stays on the same optimization basin.}
\label{fig:comp-to-pretrained}
\end{figure}
\subsubsection{Comparing Noise and Clipping}
We now examine the effects of clipping and noise separately, observing that noise alone does not degrade model performance as larger gradients can be used to recover the unwanted deviations, whereas clipping impacts model performance by slowing progress. The combination of both clipping and noise in DP-SGD is the worst case scenario, where clipping inhibits recovery from the perturbation created by noise.

In Figure~\ref{fig:five-step-loss-gn}, we observe the loss and gradient norm for clipped SGD, noisy SGD, DP-SGD and SGD, both at model initialization and later in model training. For a fixed training batch $B$, we take $t=5$ steps with clipping set to $C=1.0$, learning rate $0.1$, batch size $128$ and a noise scale of $0.4$. The mean and average values over batches of size $128$ is reported. At initialization in Figure~\ref{fig:five-step-loss-gn} a), SGD achieves the lowest loss. Both clipping and noise reduce loss improvement, but DP-SGD has significantly worse performance than any other method. The gradient norms in Figure~\ref{fig:five-step-loss-gn} b) shed light on this, with low noisy and DP-SGD gradient norms suggesting that these methods have not located a basin to optimize within quickly.

\begin{figure}[h]
\centering
\begin{tabular}{cc}
\includegraphics[width=0.45\linewidth]{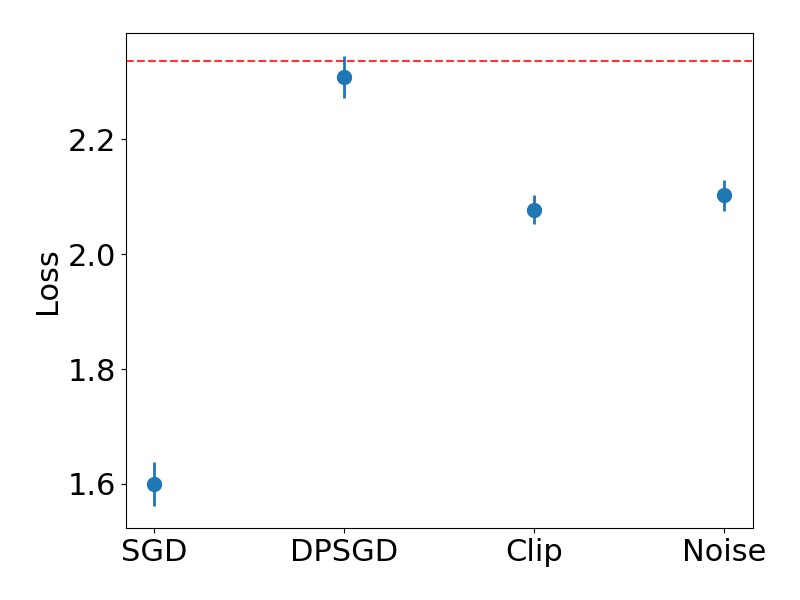} &
\includegraphics[width=0.45\linewidth]{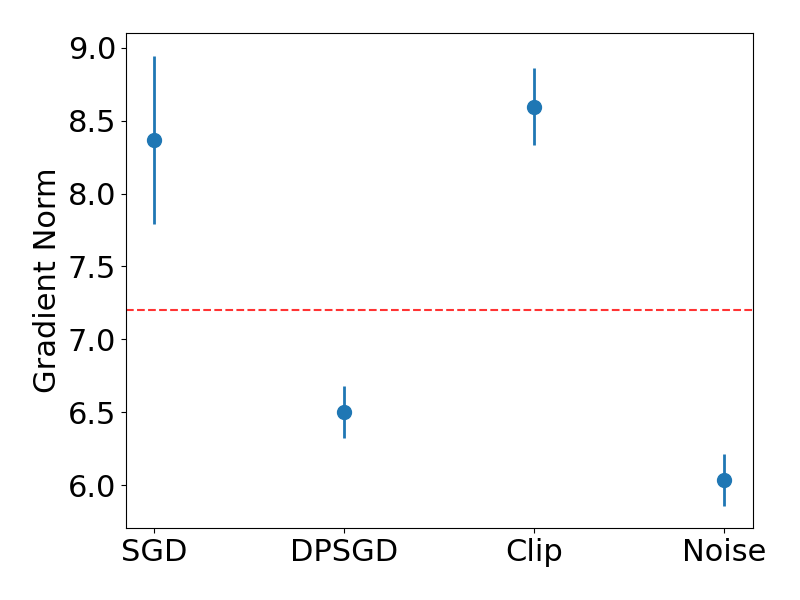}
\\
a) Loss  &b) Gradient Norm   \\
\includegraphics[width=0.45\linewidth]{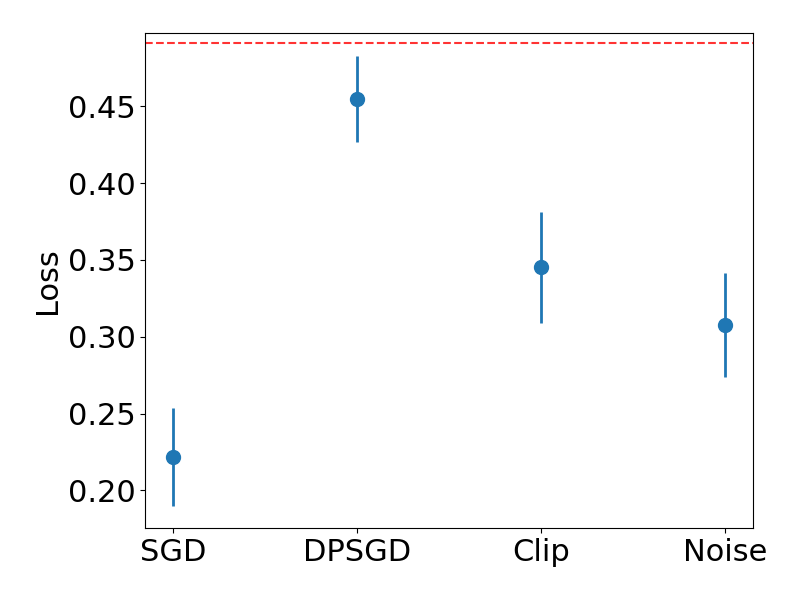} &
\includegraphics[width=0.45\linewidth]{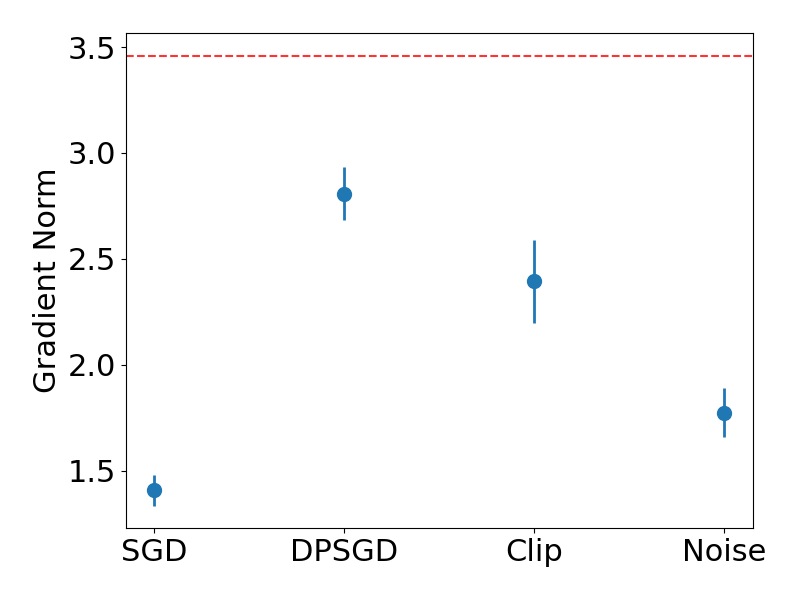}
\\
c) Loss  &d) Gradient Norm   \\
\end{tabular}
 \caption{Training loss and gradient norms after 5 training steps with a fixed training batch of size 128. The value before training is shown by the red dashed line. The top row starts training from initialization, the bottom row begins training from a model that has been trained for 10 epochs with SGD.}
\label{fig:five-step-loss-gn}
\end{figure}

When we instead begin from a model pre-trained for 10 epochs with SGD, Figure~\ref{fig:five-step-loss-gn} c) shows similar loss trends to an untrained model, but the gradient norm behavior in Figure~\ref{fig:five-step-loss-gn} d) differs. In this case, the gradient norm is largest for DP-SGD and the clipped model, with all variants having a larger gradient norm than SGD. This suggests that clipping as well as noise contributes to gradient norm increases over training, and therefore increased distortion of the true gradient at each step.  See the Appendix for a comparison of gradient norms throughout training for DP-SGD vs SGD, the larger gradient norms observed in Figure~\ref{fig:five-step-loss-gn} persist.

Figure~\ref{fig:clip-and-noise} shows the accuracy of clipping with maximum gradient norm $1.0$ and noise addition with scale $2.4$, settings equivalent to those for DP-SGD with $\epsilon=1$. If we directly compare the effects of noise and clipping for a given learning rate, batch size and number of epochs, then clipping significantly impacts training performance by slowing progress when gradient norms increase after the first few epochs, whereas noisy SGD attains the same final test accuracy as SGD. This further supports the idea that clipping plays a significant role in preventing DP-SGD from recovering from the required noise addition that would otherwise not be detrimental to performance.

\begin{figure}[htbp!]
\centering
\begin{tabular}{cc}
\includegraphics[width=0.45\linewidth]{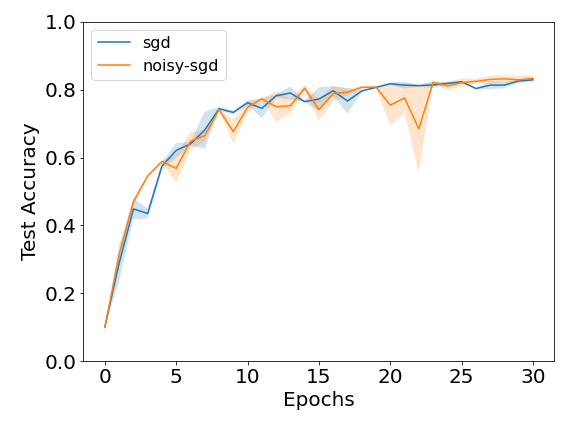} &
\includegraphics[width=0.45\linewidth]{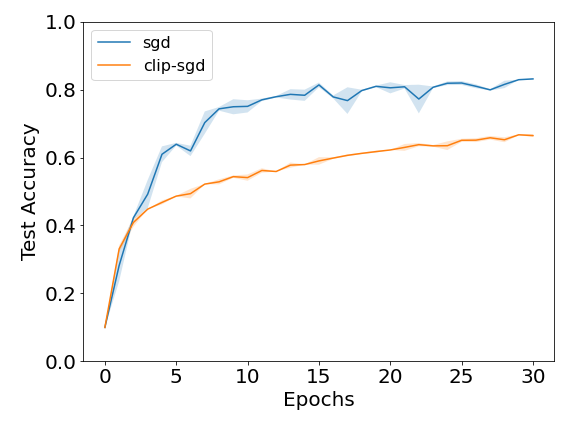}
\\
a) Noisy SGD & b) Clipped SGD
\end{tabular}
 \caption{Noisy-SGD Accuracy test accuracy comparison with $\sigma=2.4$ and Clipped-SGD with maximum gradient norm $1.0$ for Resnet18 with Cifar10.}
\label{fig:clip-and-noise}
\end{figure}

\subsection{Pruning}

Given the hypothesized benefits of pruning in both reducing the random noise added during DP-SGD training and reducing the impact of clipping, we now test these hypotheses empirically. In Figure~\ref{fig:theory-exp} we confirm that in the setting described in Section~\ref{sec:theory_example}, pruning reduces the ratio $R$, better aligning the clipped gradient
with the original gradient direction.

\begin{figure}[h]
\centering
\includegraphics[width=0.45\linewidth]{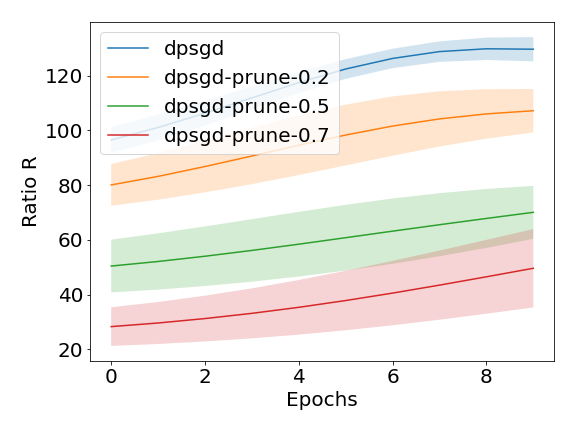}
 \caption{Ratio $R$ reduced by pruning.}
\label{fig:theory-exp}
\end{figure}

This behavior is further evidenced in Figure~\ref{fig:clip-noise-only} a) which demonstrates how pruning alleviates the reduction in accuracy caused by clipping SGD for Resnet18, even when there is no noise present. In contrast, Figure~\ref{fig:clip-noise-only} b) shows that pruning can help the network to train faster for noisy SGD on Resnet18 but not attain higher test accuracy, note that pruning also helps SGD without noise train more quickly and therefore this behavior is not unique to noisy SGD.

\begin{figure}[htbp!]
\centering
\begin{tabular}{cc}
\includegraphics[width=0.45\linewidth]{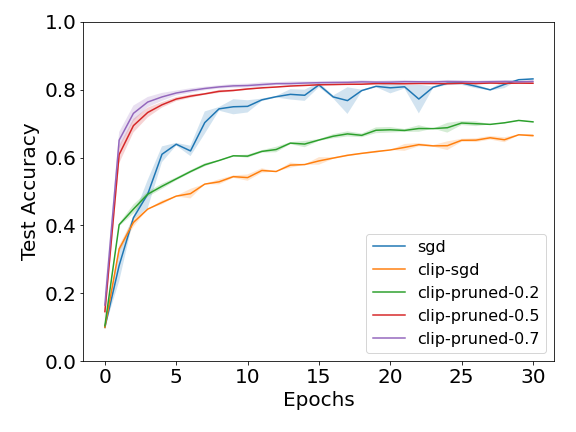} &
\includegraphics[width=0.45\linewidth]{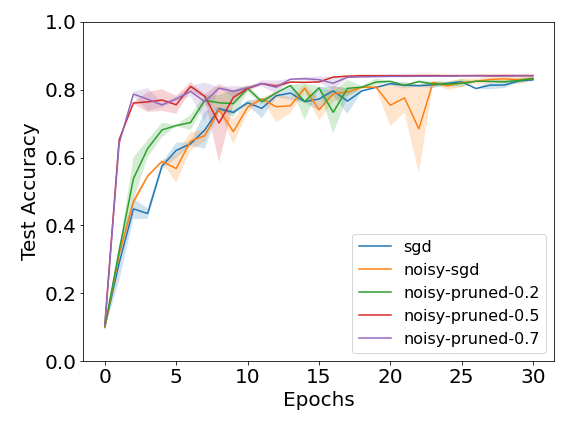}
\\
a) & b)\\

\end{tabular}
 \caption{a) Clipped-SGD Test Accuracy comparison with $C=1.0$ for Resnet18 with Cifar10 pre-training. This clipping scale corresponds to the maximum gradient norm used for $\epsilon=1$ and $\epsilon=2$ experiments. b) Noisy-SGD Accuracy comparison with $\sigma=2.4$. This noise scale corresponds to the noise level for $\epsilon=1$. Training accuracy plots show the same trends.}
\label{fig:clip-noise-only}
\end{figure}

\begin{figure}[t]
\centering
\begin{tabular}{cc}
\includegraphics[width=0.45\linewidth]{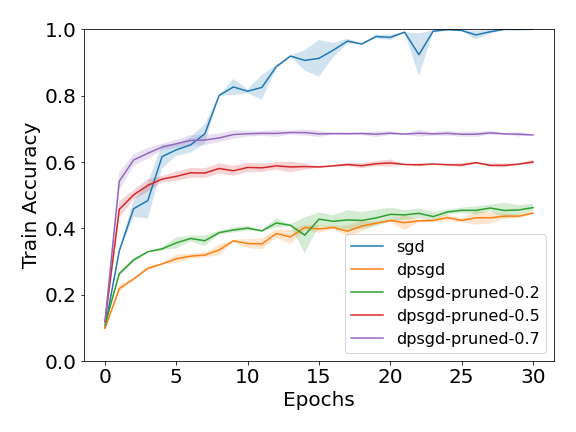} &
\includegraphics[width=0.45\linewidth]{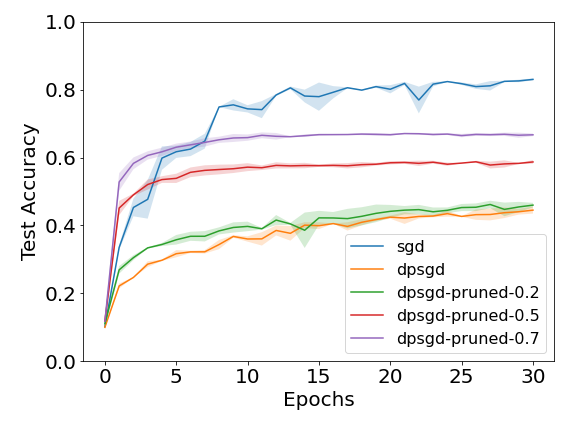}
\\
a) & b) \\
\includegraphics[width=0.45\linewidth]{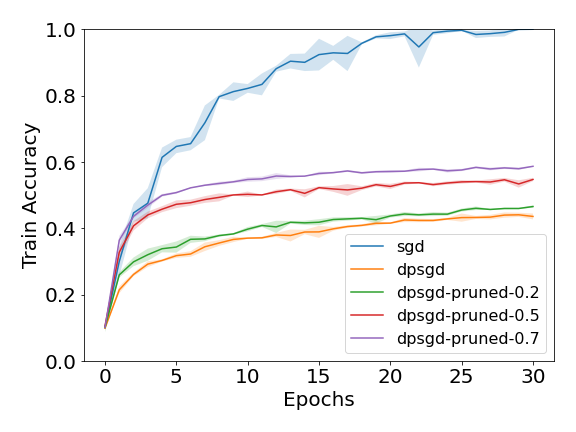} &
\includegraphics[width=0.45\linewidth]{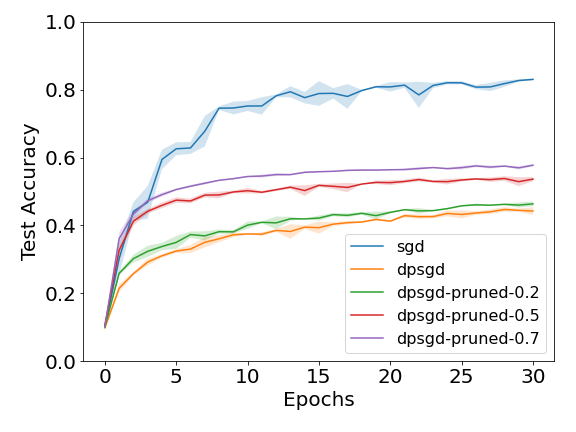}
\\
c) & d) \\
\includegraphics[width=0.45\linewidth]{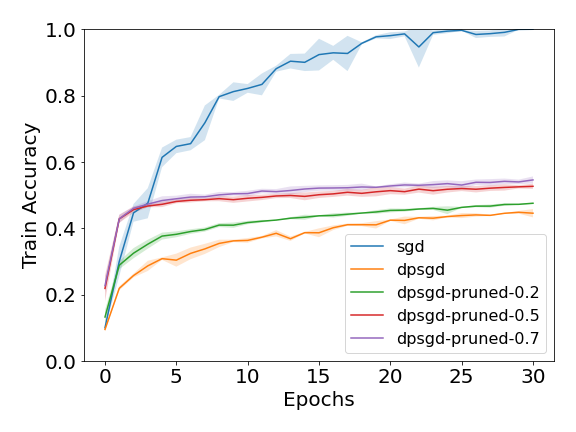} &
\includegraphics[width=0.45\linewidth]{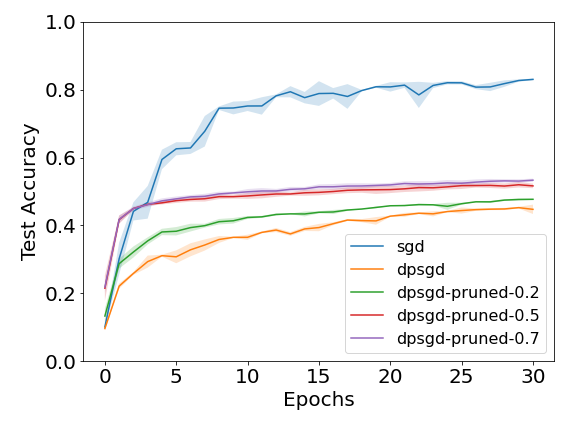}
\\
e) & f) \\
\end{tabular}
 \caption{DP-SGD Test Accuracy comparison with $\epsilon=1$ for Resnet18. Top row: magnitude pruning via non-private Cifar10 Phase 1. Middle row: magnitude pruning via non-private Cifar10 using 5$\%$ of Cifar10. Bottom row: magnitude pruning via Cifar100.}
\label{fig:test-accs-pruning}
\end{figure}

Given, these results, the  dramatic improvement in test accuracy for DP-SGD with pruning shown Figure~\ref{fig:test-accs-pruning}is unsurprising, demonstrating that pruning increases DP-SGD accuracy, with the same trend as Figure~\ref{fig:theory-exp} where 0.7 pruning has the strongest performance. The large improvement of around $23\%$ seen in Figure~\ref{fig:test-accs-pruning} a) and b) using non-private pruning implies that if pruning can be effectively implemented in a privacy preserving manner the benefits of pruning can be significant. Figure~\ref{fig:test-accs-pruning} e) and f) shows an example of how to achieve this, by pruning using  Cifar100. See~\citep{luo2021scalable} for further evidence of the potential test accuracy benefits of pruning for DP-SGD test accuracy. Another notable aspect addressed by pruning is that of DP-SGD trained models residing significantly further from the origin than SGD models (e.g. $1.7\times$). In the Appendix we see that pruning reduces this difference, with parameter distributions for pruned DP-SGD models more closely replicating those of SGD models. See the Appendix for further results, including ablation studies.

\section{RELATED WORKS}
\label{sec:related}

Since its initial proposal, Differentially Private Stochastic Gradient Descent has been studied for its convergence properties~\citep{abadi2016deep,bassily2014private}. More recent theoretical works have studied clipping for its impact on convergence~\citep{xiao2023theory,chen2020understanding}.
Theoretical  approaches of sparsifying gradients to improve the performance of DPSGD have been developed, where sparsification is carried out individually for each gradient at runtime~\citep{zhou2021bypassing,yu2021not}. In contrast to gradient sparsification, pruning can be viewed as a one time pre-sparsification of both the network and gradients. Pruning of the normalization layers of Resnet has been found to improve performance of DPSGD~\citep{luo2021scalable}.

In contrast to these works, we use mechanisms such as Linear mode connectivity~\citep{frankle2020linear} to study the loss basins and make fine grained comparisons of DPSGD models with SGD models. We demonstrate the individual impacts of clipping and noise and the low dimensionality of loss basins, and provide theoretical justifications for pruning.

\section{CONCLUSION}

This paper has provided several new insights theoretically and experimentally into the operation of DP-SGD. While it does not quite solve the problem of  training private models that are almost as good as non-private ones, the insights from this research should help us to develop better private training algorithms in future.

\bibliography{references}

\begin{thebibliography}{}

\bibitem[Abadi et~al., 2016]{abadi2016deep}
Abadi, M., Chu, A., Goodfellow, I., McMahan, H.~B., Mironov, I., Talwar, K.,
  and Zhang, L. (2016).
\newblock Deep learning with differential privacy.
\newblock In {\em Proceedings of the 2016 ACM SIGSAC conference on computer and
  communications security}, pages 308--318.

\bibitem[Arora et~al., 2018]{arora2018optimization}
Arora, S., Cohen, N., and Hazan, E. (2018).
\newblock On the optimization of deep networks: Implicit acceleration by
  overparameterization.
\newblock In {\em International Conference on Machine Learning}, pages
  244--253. PMLR.

\bibitem[Bassily et~al., 2014]{bassily2014private}
Bassily, R., Smith, A., and Thakurta, A. (2014).
\newblock Private empirical risk minimization: Efficient algorithms and tight
  error bounds.
\newblock In {\em 2014 IEEE 55th annual symposium on foundations of computer
  science}, pages 464--473. IEEE.

\bibitem[Carlini et~al., 2019]{carlini2019secret}
Carlini, N., Liu, C., Erlingsson, {\'U}., Kos, J., and Song, D. (2019).
\newblock The secret sharer: Evaluating and testing unintended memorization in
  neural networks.
\newblock In {\em 28th USENIX Security Symposium (USENIX Security 19)}, pages
  267--284.

\bibitem[Chen et~al., 2020]{chen2020understanding}
Chen, X., Wu, S.~Z., and Hong, M. (2020).
\newblock Understanding gradient clipping in private sgd: A geometric
  perspective.
\newblock {\em Advances in Neural Information Processing Systems},
  33:13773--13782.

\bibitem[Dwork, 2006]{Dwork06differentialprivacy}
Dwork, C. (2006).
\newblock Differential privacy.
\newblock In {\em Automata, Languages and Programming}, pages 1--12. ICALP.

\bibitem[Dwork et~al., 2006]{dwork2006calibrating}
Dwork, C., McSherry, F., Nissim, K., and Smith, A. (2006).
\newblock Calibrating noise to sensitivity in private data analysis.
\newblock In {\em Theory of Cryptography: Third Theory of Cryptography
  Conference, TCC 2006, New York, NY, USA, March 4-7, 2006. Proceedings 3},
  pages 265--284. Springer.

\bibitem[Frankle et~al., 2020]{frankle2020linear}
Frankle, J., Dziugaite, G.~K., Roy, D., and Carbin, M. (2020).
\newblock Linear mode connectivity and the lottery ticket hypothesis.
\newblock In {\em International Conference on Machine Learning}, pages
  3259--3269. PMLR.

\bibitem[Fredrikson et~al., 2015]{fredrikson2015model}
Fredrikson, M., Jha, S., and Ristenpart, T. (2015).
\newblock Model inversion attacks that exploit confidence information and basic
  countermeasures.
\newblock In {\em Proceedings of the 22nd ACM SIGSAC conference on computer and
  communications security}, pages 1322--1333.

\bibitem[Ganesh et~al., 2023]{ganesh2023public}
Ganesh, A., Haghifam, M., Nasr, M., Oh, S., Steinke, T., Thakkar, O., Thakurta,
  A.~G., and Wang, L. (2023).
\newblock Why is public pretraining necessary for private model training?
\newblock In {\em International Conference on Machine Learning}, pages
  10611--10627. PMLR.

\bibitem[Garipov et~al., 2018]{garipov2018loss}
Garipov, T., Izmailov, P., Podoprikhin, D., Vetrov, D.~P., and Wilson, A.~G.
  (2018).
\newblock Loss surfaces, mode connectivity, and fast ensembling of dnns.
\newblock {\em Advances in neural information processing systems}, 31.

\bibitem[He et~al., 2016]{he2016deep}
He, K., Zhang, X., Ren, S., and Sun, J. (2016).
\newblock Deep residual learning for image recognition.
\newblock In {\em Proceedings of the IEEE conference on computer vision and
  pattern recognition}, pages 770--778.

\bibitem[Kawaguchi, 2016]{kawaguchi2016deep}
Kawaguchi, K. (2016).
\newblock Deep learning without poor local minima.

\bibitem[Krizhevsky et~al., 2009]{krizhevsky2009learning}
Krizhevsky, A., Hinton, G., et~al. (2009).
\newblock Learning multiple layers of features from tiny images.

\bibitem[LeCun et~al., 1998]{lecun1998gradient}
LeCun, Y., Bottou, L., Bengio, Y., and Haffner, P. (1998).
\newblock Gradient-based learning applied to document recognition.
\newblock {\em Proceedings of the IEEE}, 86(11):2278--2324.

\bibitem[Li et~al., 2022]{li2022does}
Li, X., Liu, D., Hashimoto, T.~B., Inan, H.~A., Kulkarni, J., Lee, Y.-T., and
  Guha~Thakurta, A. (2022).
\newblock When does differentially private learning not suffer in high
  dimensions?
\newblock {\em Advances in Neural Information Processing Systems},
  35:28616--28630.

\bibitem[Luo et~al., 2021]{luo2021scalable}
Luo, Z., Wu, D.~J., Adeli, E., and Fei-Fei, L. (2021).
\newblock Scalable differential privacy with sparse network finetuning.
\newblock In {\em Proceedings of the IEEE/CVF Conference on Computer Vision and
  Pattern Recognition}, pages 5059--5068.

\bibitem[Nasr et~al., 2019]{nasr2019comprehensive}
Nasr, M., Shokri, R., and Houmansadr, A. (2019).
\newblock Comprehensive privacy analysis of deep learning: Passive and active
  white-box inference attacks against centralized and federated learning.
\newblock In {\em 2019 IEEE symposium on security and privacy (SP)}, pages
  739--753. IEEE.

\bibitem[Papernot et~al., 2021]{papernot2021tempered}
Papernot, N., Thakurta, A., Song, S., Chien, S., and Erlingsson, {\'U}. (2021).
\newblock Tempered sigmoid activations for deep learning with differential
  privacy.
\newblock In {\em Proceedings of the AAAI Conference on Artificial
  Intelligence}, volume~35, pages 9312--9321.

\bibitem[Shokri et~al., 2017]{shokri2017membership}
Shokri, R., Stronati, M., Song, C., and Shmatikov, V. (2017).
\newblock Membership inference attacks against machine learning models.
\newblock In {\em 2017 IEEE symposium on security and privacy (SP)}, pages
  3--18. IEEE.

\bibitem[Xiao et~al., 2017]{fmnist}
Xiao, H., Rasul, K., and Vollgraf, R. (2017).
\newblock Fashion-mnist: a novel image dataset for benchmarking machine
  learning algorithms.
\newblock {\em CoRR}, abs/1708.07747.

\bibitem[Xiao et~al., 2023]{xiao2023theory}
Xiao, H., Xiang, Z., Wang, D., and Devadas, S. (2023).
\newblock A theory to instruct differentially-private learning via clipping
  bias reduction.
\newblock In {\em 2023 IEEE Symposium on Security and Privacy (SP)}, pages
  2170--2189. IEEE Computer Society.

\bibitem[Yousefpour et~al., 2021]{yousefpour2021opacus}
Yousefpour, A., Shilov, I., Sablayrolles, A., Testuggine, D., Prasad, K.,
  Malek, M., Nguyen, J., Ghosh, S., Bharadwaj, A., Zhao, J., et~al. (2021).
\newblock Opacus: User-friendly differential privacy library in pytorch.
\newblock {\em arXiv preprint arXiv:2109.12298}.

\bibitem[Yu et~al., 2021]{yu2021not}
Yu, D., Zhang, H., Chen, W., and Liu, T.-Y. (2021).
\newblock Do not let privacy overbill utility: Gradient embedding perturbation
  for private learning.
\newblock In {\em International Conference on Learning Representations}.

\bibitem[Zhou et~al., 2021]{zhou2021bypassing}
Zhou, Y., Wu, S., and Banerjee, A. (2021).
\newblock Bypassing the ambient dimension: Private {\{}sgd{\}} with gradient
  subspace identification.
\newblock In {\em International Conference on Learning Representations}.

\end{thebibliography}

\onecolumn
\section{Appendix}
\begin{proof}[Proof of Theorem \ref{thm:clipped_grad_alignment}]
Decompose $\vct{g}_i = a_i \frac{\bar{\vct{g}}}{\|\bar{\vct{g}}\|} + b_i \vct{u}_i$, where $\vct{u}_i$ is a unit vector orthogonal to $\bar{\vct{g}}$. Then the desired inequality becomes
\begin{align}
    \frac{1}{n} \sum_{i=1}^n \frac{a_i}{\sqrt{a_i^2 + b_i^2}} \geq 1 - \frac{\frac{1}{n}\sum_{i=1}^n a_i^2 + b_i^2 - \bar{a}^2}{2\bar{a}^2}
\end{align}
where $\bar{a}$ is shorthand for $\frac{1}{n} \sum_{i=1}^n a_i$, and we know $\bar{a} > 0$.

Define the function $f$ by
\begin{align}
    f(a_1, \dots, a_n, b_1, \dots, b_n) &= \frac{2}{n}  \sum_{i=1}^n \frac{a_i}{\sqrt{a_i^2 + b_i^2}} + \frac{1}{n \bar{a}^2} \left(\sum_{i=1}^n a_i^2 + b_i^2\right)
\end{align}

For any $0 < \alpha < 3\alpha < \beta, 0 < \gamma$, let $J_\alpha = \{(a_1, \dots, a_n, b_1, \dots, b_n) \in \R^{2n}:  \bar{a} \geq \alpha\}$ and let $I_{\alpha, \beta, \gamma} = J_\alpha \cap \left(([-\beta, -\alpha] \cup [\alpha, \beta])^n \times [-\gamma, \gamma ]^n\right)$. We claim that $f \geq 3$ on $I_{\alpha, \beta, \gamma}$. Since $I_{\alpha, \beta, \gamma}$ is compact, $f$ has a minimum on $I_{\alpha, \beta, \gamma}$. It suffices to consider the critical points $f$ and the boundary of $I_{\alpha, \beta, \gamma}$.

We first find the critical points of $f$. For such critical points, calculate that
\begin{align}
    \frac{\partial f}{\partial a_j} &= \frac{2}{n} \frac{b_j^2}{(a_j^2 + b_j^2)^{\frac{3}{2}}} + \frac{2a_j}{n\bar{a}^2} - \frac{2}{n^2\bar{a}^3} \left(\sum_{i=1}^n a_i^2 + b_i^2\right) = 0  \label{df_da} \\
    \frac{\partial f}{\partial b_j} &= -\frac{2}{n} \frac{a_jb_j}{(a_j^2 + b_j^2)^{\frac{3}{2}}} + \frac{2 b_j}{n \bar{a}^2} = 0 \label{df_db} \
\end{align}
Suppose that $b_{k_1} = b_{k_2} = 0$ for some $k_1, k_2$. Then we must have
\begin{align}
   \frac{2a_{k_1}}{n\bar{a}^2} - \frac{2}{n^2\bar{a}^3} \left(\sum_{i=1}^n a_i^2 + b_i^2\right) = \frac{2a_{k_2}}{n\bar{a}^2} - \frac{2}{n^2\bar{a}^3} \left(\sum_{i=1}^n a_i^2 + b_i^2\right)
\end{align}
which implies that $a_{k_1} = a_{k_2} = A_1$ for some constant $A_1$.

Suppose that for some $k, b_k \neq 0$. Then equation \ref{df_db} implies that
\begin{equation}
    \label{norm_orig}
    \bar{a}^2 a_k = (a_k^2 + b_k^2)^{\frac{3}{2}}
\end{equation}
which also implies
\begin{equation}
    b_k^2 = \bar{a}^{\frac{4}{3}}a_k^{\frac{2}{3}} - a_k^2
\end{equation}
Substituting into Equation \ref{df_da},
\begin{align}
    0 &= \frac{2}{n} \cdot \frac{\bar{a}^{\frac{4}{3}} a_k^{\frac{2}{3}} - a_k^2}{\bar{a}^2 a_k} + \frac{2a_k}{n\bar{a}^2} - \frac{2}{n^2\bar{a}^3} \left(\sum_{i=1}^n a_i^2 + b_i^2\right) \\
    &= \frac{2}{n \bar{a}^2}\left( \frac{\bar{a}^{\frac{4}{3}} a_k^{\frac{2}{3}}}{a_k} \right) - \frac{2}{n^2\bar{a}^3} \left(\sum_{i=1}^n a_i^2 + b_i^2\right) \\
    &= \frac{2}{n \bar{a}^{\frac{2}{3}} a_k^{\frac{1}{3}}} - \frac{2}{n^2\bar{a}^3} \left(\sum_{i=1}^n a_i^2 + b_i^2\right)
\end{align}
It follows that for any $k_1, k_2$ with $b_{k_1}, b_{k_2} \neq 0$, we must have $a_{k_1} = a_{k_2} = A_2$ for some constant $A_2$. Moreover, we must have $\frac{2}{n \bar{a}^{\frac{2}{3}} A_2^{\frac{1}{3}}} = \frac{2A_1}{n\bar{a}^2}$ which gives
\begin{equation} \label{dual_relation}
   A_1^3 A_2 = \bar{a}^4
\end{equation}

Consider two cases. If all the $a_i$ are equal, then $a_i = A_1 = A_2 = \bar{a} > 0$, and also
\begin{equation}
    b_i^2 = \bar{a}^{\frac{4}{3}}a_i^{\frac{2}{3}} - a_i^2 = 0,
\end{equation}
so $b_i = 0$ for all $i$. One can check that $f$ at such a point is equal to 3.

Now suppose for the sake of contradiction not all the $a_i$ are equal. The one of them, WLOG let it be $a_1$, satisfies $a_1 > \bar{a}$. Then $a_1 \bar{a}^2 < (a_1^2 + b_1^2)^{\frac{3}{2}}$, so the only way Equation \ref{df_db} is satisfied is if $b_1 = 0$. But then $a_1 = A_1$. It follows from Equation \ref{dual_relation} $0 < A_2 < \bar{a} < A_1$.

Now let $n_1$ be the number of $a_i$ equal to $A_1$ and $A_2$. We must have
\begin{align}
    \bar{a} &= \frac{n_1}{n} A_1 + (1 - \frac{n_1}{n})A_2 \\
    &= \frac{n_1}{n} A_1 + (1 - \frac{n_1}{n}) \frac{\bar{a}^4}{A_1^3}
\end{align}
Set $\lambda = \frac{A_1}{\bar{a}} > 1$. Then
\begin{align}
    1 &= \frac{n_1}{n} \lambda + (1 - \frac{n_1}{n}) \frac{1}{\lambda^3} \\
    (\lambda^3 - 1) &= (\lambda^4 - 1) \frac{n_1}{n} \\
    \frac{n_1}{n} &= \frac{\lambda^3 - 1}{\lambda^4 - 1}
\end{align}
Now evaluating $f$,
\begin{align}
    f(a_1, \dots, a_n, b_1, \dots, b_n) &= \frac{2}{n}\left(n_1 + n_2\frac{A_2}{A_2^{\frac{1}{3}} \bar{a}^{\frac{2}{3}}} \right) + \frac{1}{n \bar{a}^2} \left(n_1A_1^2 + A_2^{\frac{2}{3}} \bar{a}^{\frac{4}{3}}\right) \\
    &= 2\left(\frac{\lambda^3 - 1}{\lambda^4 - 1} + (1 - \frac{\lambda^3 - 1}{\lambda^4 - 1})\frac{1}{\lambda^2}\right) + \left(\frac{\lambda^3 - 1}{\lambda^4 - 1}\lambda^2 + (1 - \frac{\lambda^3 - 1}{\lambda^4 - 1})\frac{1}{\lambda^2}\right)
\end{align}
It is not hard to show that the RHS is greater than $3$ for $\lambda > 1$. Thus we have shown that $f \geq 3$ at all critical points.

It remains to check the boundary points. Notice that the dominant term in $f $ as $|b_i| \to \infty$ is $b_i^2$. Hence using sufficiently large $\gamma$ guarantees that no point with $b_i = \gamma$ for some $i$ can be a minimum. On the other hand, suppose of the $a_i$ take a boundary value. We consider a few cases:

Case 1: $a_i = -\beta$. If $b_i \neq 0$, then $\frac{\partial f}{\partial b_i} \neq 0$, so we can adjust $b_i$ to decrease $f$. If $b_i = 0$, then \ref{df_da} shows that increasing $a_i$ will decrease the value of $f$.

Case 2: If the only boundary values are of the form $a_i = \pm \alpha$ observe that $f$ is scale-invariant, so scaling all inputs by $1 + \delta$ shows that $f$ has the same value at some interior point of $I_{\alpha, \beta, \gamma}$, which we already know is not a minimum or greater than 3.

Case 3: $a_i$ on the boundary are only of the form $a_i = \beta$. A similar argument as the previous case but scaling down works.

Case 4: Some $a_i = \alpha$ and some $a_j = \beta$. But then from Equation \ref{df_da} we see that $\frac{\partial f}{\partial a_i} < \frac{\partial f}{\partial a_j}$, so either $\frac{\partial f}{\partial a_i} < 0$ or $\frac{\partial f}{\partial a_j} > 0$. This means that either increasing $a_i$ or decreasing $a_j$ decreases the value of $f$, so this cannot be a minimum.

Case 5: Some $a_i = -\alpha$ and some $a_j = \beta$. If $b_i \neq 0$ then $\frac{\partial f}{\partial b_i} \neq 0$ by \ref{df_db} and we can adjust $b_i$ appropriately to decrease $f$. A similar argument handles the case that $b_j \neq 0$.

Suppose $b_i = b_j = 0$. Set $a_i' = \alpha, a_j' = \beta - 2 \alpha$. Then
\begin{align}
    f(a_1 ,\dots, a_i', \dots, a_j', \dots a_n, b_1, \dots, b_n) - f(a_1, \dots, a_n, b_1, \dots b_n) &= \frac{2}{n}(1 - (-1)) + \frac{1}{n \bar{a}^2}((\beta - 2\alpha)^2 - \beta^2)  \\
    &= \frac{4}{n}\left(1 - \frac{\beta(\beta - \alpha)}{\bar{a}^2}\right)
\end{align}
Taking $\beta > n\alpha$ guarantees that $\beta > \beta - \alpha > \bar{a}$, so for $\beta$ sufficiently large the RHS is less than 0, showing that the original point is not a minimum.

The final boundary condition is if $\bar{a} = \alpha$. We have already considered the condition where any of the $a_i$ are equal to $\pm \alpha$, so the only way this is possible if there exists some $i$ s.t. $a_i < -\alpha$. If $b_i \neq 0$, we can decrease $|b_i|$ to decrease $f$ by \ref{df_db}. If $b_i = 0$, then \ref{df_da} shows that increasing $\alpha_i$ will decrease the value  of $f$. Thus such a point cannot be a minimum. This proves that $f \geq 3$ on $I_{\alpha, \beta, \gamma}$.

Now since $\alpha, \beta, \gamma$ work for any $\beta, \gamma$ sufficiently large, we conclude that $f \geq 3$ on $\{(a_1, \dots, a_n, b_1, \dots, b_n) \in \R^{2n}:  \bar{a} > 0\} \cap ((\R \setminus \{0\})^n \times \R^n)$. We can extend to the possibility that some of the $a_i$ are zero by continuity. This completes the proof.

\end{proof}

\begin{proof} [Proof of \ref{thm:magnitude_pruning}]
Collect the input examples into a matrix $\mtx{X} = [x^{(1}, \dots, x^{(n)}]$ and let $\mtx{Z} = f(\mtx{X}) \in \R^{1 \times n}$ be the matrix of model outputs on the dataset. Then the gradients are
\begin{align*}
    \frac{\partial \mathcal{L}}{\partial \mtx{W}_1} &= \mtx{W}_2^{\top} \frac{\partial \mathcal{L}}{\partial \mtx{Z}} \mtx{X}^{\top} \\
    \frac{\partial \mathcal{L}}{\partial \mtx{W}_2} &= \frac{\partial \mathcal{L}}{\partial \mtx{Z}} \mtx{X}^{\top} \mtx{W}_1^{\top} \\
\end{align*}
In addition, observe that using the rule of gradient flow,
\begin{align*}
    \frac{d}{dt} (\mtx{W}_1 \mtx{W}_1^{\top}) &= \mtx{W}_1 \left(\frac{d \mtx{W}_1}{\partial t}\right)^{\top} + \frac{d \mtx{W}_1}{\partial t} \mtx{W}_1^{\top} \\
    &= -\mtx{W}_1 \left(\frac{\partial \mathcal{L}}{\partial \mtx{W}_1}\right)^{\top} - \frac{\partial \mathcal{L}}{\partial \mtx{W}_1} \mtx{W}_1^{\top} \\
    \frac{d}{dt} (\mtx{W}_2^{\top} \mtx{W}_2) &= \mtx{W}_2 \left(\frac{d \mtx{W}_2}{\partial t}\right)^{\top} + \frac{d \mtx{W}_2}{\partial t} \mtx{W}_2^{\top} \\
    &= -\left(\frac{\partial \mathcal{L}}{\partial \mtx{W}_2}\right)^{\top} \mtx{W}_2 - \mtx{W}_2^{\top} \frac{\partial \mathcal{L}}{\partial \mtx{W}_2}.
\end{align*}
Substituting, we see that $\frac{d}{dt} (\mtx{W}_1 \mtx{W}_1^{\top}) = \frac{d}{dt} (\mtx{W}_2^{\top} \mtx{W}_2)$. Since we assumed that that $\mtx{W}_1 \mtx{W}_1^{\top} = \mtx{W}_2^{\top} \mtx{W}_2$, it follows that equality holds throughout training.

Now the model is linear in $\mtx{W} = \mtx{W}_2 \mtx{W}_1$, so by well known results on linear regression, the optimal model $\mtx{W}_2^* \mtx{W}_1^*$ on the entire dataset satisfies
\begin{align*}
    \mtx{W}_2^* \mtx{W}_1^* &= \E_{\vct{x} \sim \mathcal{D}} [y\vct{x}^{\top}] \left(\E_{\vct{x} \sim \mathcal{D}} [\vct{x}\vct{x}^{\top}]\right)^{\dagger} \\
    &= \vct{v}^{\top} (\vct{v}\vct{v}^{\top} + \sigma^2 \mtx{I})^{\dagger} \\
    &= \frac{\|\vct{v}\|}{\|\vct{v}\|^2 + \sigma^2} \vct{v}^{\top}
\end{align*}
Combined with the fact that $\mtx{W}_1^* (\mtx{W}_1^*)^{\top} = (\mtx{W}_2^*)^{\top} \mtx{W}_2^*$, it follows that $\mtx{W}_1^*, \mtx{W}_2^*$ are both rank 1 matrices with singular value $\sqrt{\frac{\|\vct{v}\|}{(\|\vct{v}\|^2 + \sigma^2)}}$, and the left singular vector of $\mtx{W}_1$ corresponds to the right singular vector of $\mtx{W}_2$ In particular,
\begin{align}
    \mtx{W}_1^* &= \sqrt{\frac{1}{\|\vct{v}\| (\|\vct{v}\|^2 + \sigma^2)}}\vct{\alpha} \vct{v}^{\top} \\
    \mtx{W}_2^* &= \sqrt{ \frac{\|\vct{v}\|}{(\|\vct{v}\|^2 + \sigma^2)}}\vct{\alpha}^{\top}
\end{align}
for some unit vector $\alpha \in \R^m$. From this expression it is easy to see that $\mtx{W}_n^* = \mtx{0}$.
\end{proof}

\section{Experimental Details}
\label{sec:app_exp_details}
We now outline further experimental results and provide further details about experiments included in the main body of text.

\textbf{Previous Experiment Details:} In Figure~\ref{fig:test-accs-pruning} c), the Cifar100 Resnet18 model used to prune the model was trained using SGD with data augmentation (horizontal flipping and random cropping), momentum=0.9, learning rate 0.02 and batch size 512 for 100 epochs. For DP-SGD, unless $\sigma$ is explicitly stated, then $\sigma$ has been calculated automatically by the Opacus package for the fixed $\epsilon$ and $\delta$ values reported using the \texttt{make\_private\_with\_epsilon} function.

\textbf{Additional Models and Datasets:} In this section, in addition to Cifar10 and Cifar100 as used previously, we also provide results using the MNIST~\citep{lecun1998gradient} and FashionMNIST~\citep{fmnist} datasets.  Table~\ref{tab:hyps} shows hyperparameters for SGD and DP-SGD variants of the model and dataset combinations included in this Appendix. Hyperparameters described in Table~\ref{tab:hyps} were selected via hyperparameter sweeps for DP-SGD and SGD with learning rates $\{0.01,0.05,0.1,0.5,1.0\}$ and batch sizes $\{32,64,128,512,1024\}$. Pruned variants were then trained using the same fixed hyperparameters.

\begin{table}[htbp!]
\centering
\begin{tabular}{cc|cc|cc}
\multirow{2}{*}{Model}       & \multirow{2}{*}{Dataset} & \multicolumn{2}{c|}{SGD} & \multicolumn{2}{c}{DP-SGD} \\ \cline{3-6}
 &                          & $\eta$    & Batch Size   & $\eta$     & Batch Size    \\ \hline
\multirow{2}{*}{CNN (LeNet)}
 & MNIST                    & 0.1       & 128          & 1.0        & 512           \\
 & FashionMNIST             & 0.1       & 64           & 0.5        & 512
\end{tabular}
\caption{Hyperparameter choices for experiments included in this section}
\label{tab:hyps}
\end{table}

\section{Further Experiments}
\label{sec:further_exps}

\textbf{Further results for Resnet18:} Figures~\ref{fig:test-accs-eps2} demonstrates the performance of pruning for Cifar10 and Resnet18 with $\epsilon=2$. As with $\epsilon=1$, pruning using non-private pre-training on Cifar10 results in the highest final test accuracy for DP-SGD with a pruning proportion of 0.7. As before, pruning using Cifar100 or 5\% of Cifar10 also improves the test accuracy of the model.

\textbf{Pruning percentages:} Figure~\ref{fig:pruning-strength} shows that for Cifar10 and Resnet18, increasing the pruning proportion past 0.7 does not further improve performance.

\begin{figure}[htbp!]
\centering
\begin{tabular}{cc}
\includegraphics[width=0.3\linewidth]{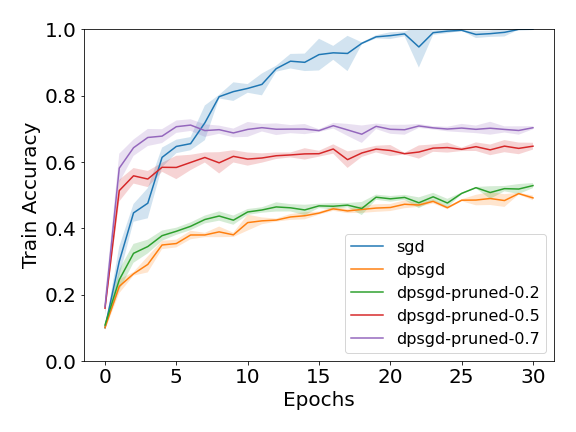} &
\includegraphics[width=0.3\linewidth]{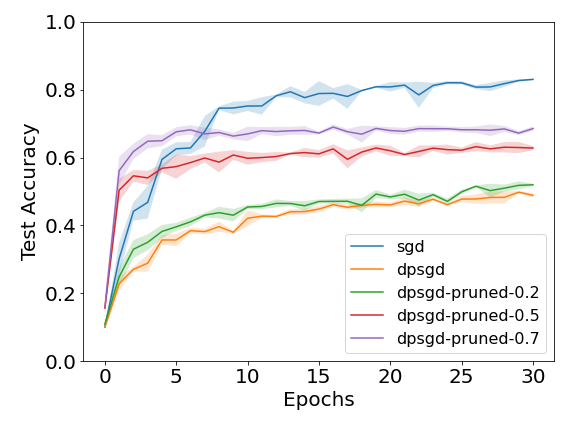} \\
a) & b)\\
\includegraphics[width=0.3\linewidth]{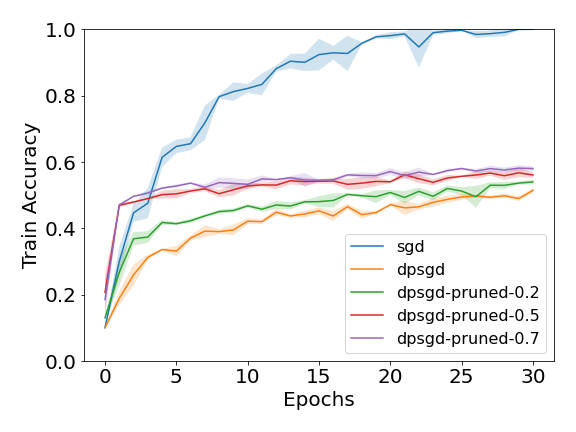}&
\includegraphics[width=0.3\linewidth]{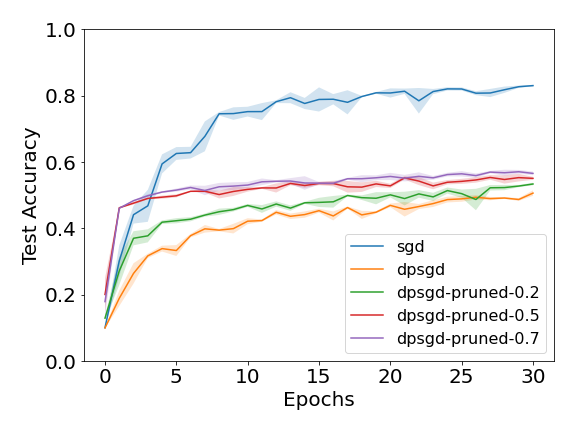}
\\
 c) & d)\\
\includegraphics[width=0.3\linewidth]{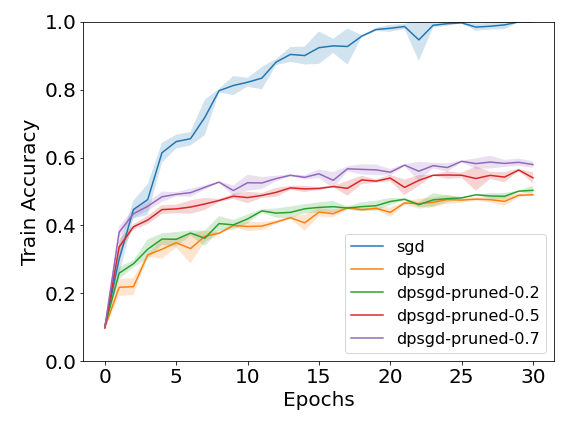} &
\includegraphics[width=0.3\linewidth]{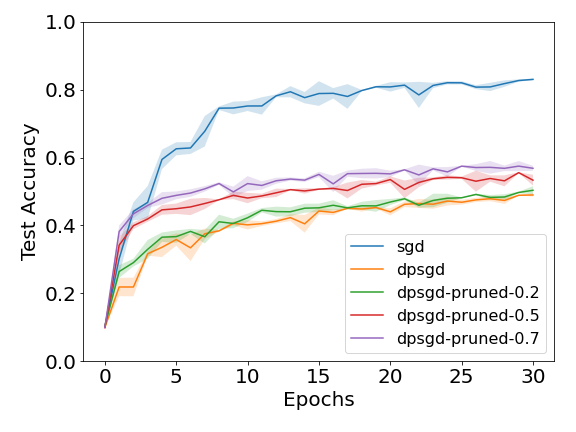}
\\
e)  & f)  \\
\end{tabular}
 \caption{DPSGD Test Accuracy comparison with $\epsilon=2$ for Resnet18 with Cifar10. a) and b) show pruning via non-private training with Cifar10 c) and d) show pruning via non-private training with 5$\%$ of Cifar10. e) and f) show pruning via training with Cifar100.}
\label{fig:test-accs-eps2}
\end{figure}

\begin{figure}[htbp!]
\centering
\begin{tabular}{cc}
\includegraphics[width=0.30\linewidth]{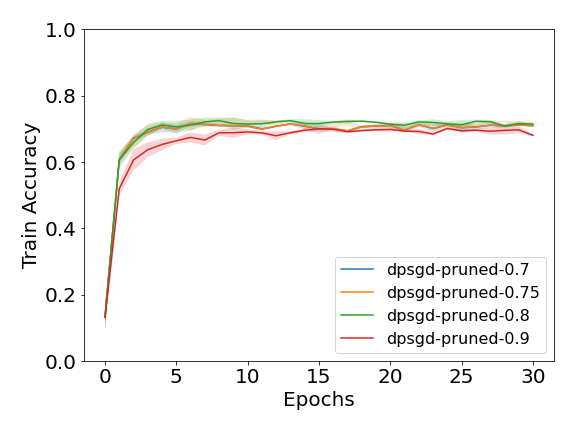} &
\includegraphics[width=0.3\linewidth]{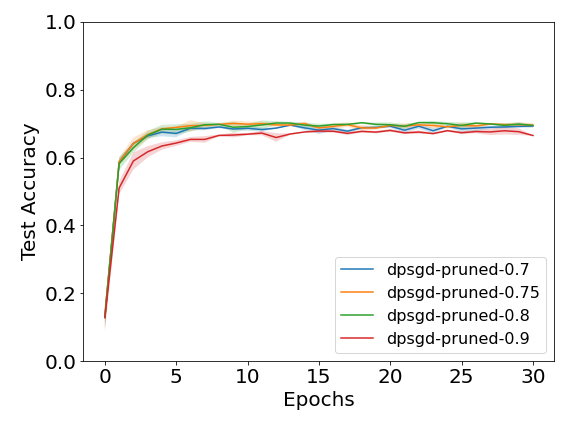}
\\
\end{tabular}
 \caption{Sweep over higher pruning percentages for Cifar10 and Resnet18.}
\label{fig:pruning-strength}
\end{figure}

\textbf{Parameter and Gradient Norm Distributions:} Gradient norms are larger during training for DP-SGD in comparison to SGD, as seen in Figure~\ref{fig:gn-over-epochs}. This will result in more severe clipping of the gradient than would otherwise have been necessary. Figure~\ref{fig:param-dists} shows that pruned DP-SGD models better replicate the parameter distribution of SGD models. DP-SGD models have significantly fewer parameters close to 0 than SGD models, their L2 distance from the origin is correspondingly around $1.7\times$ larger. Pruning reduces this to $1.2\times$ larger.

\begin{figure}[ht]%
 \centering
 \subfloat[]{\includegraphics[width=0.25\linewidth]{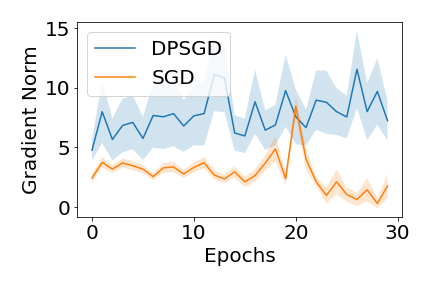}}\\
 \subfloat[]{\includegraphics[width=0.25\linewidth]{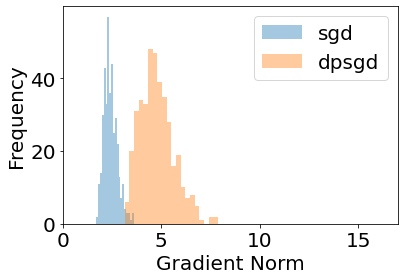}}
 \subfloat[]{\includegraphics[width=0.25\linewidth]{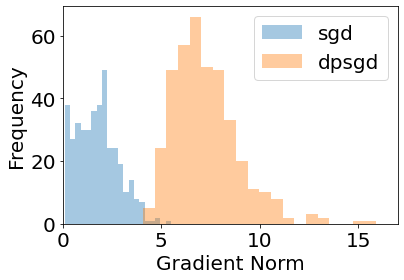}} \\
 \caption{Top row: Average gradient norm over training data batches of size 128 for SGD vs DP-SGD with Resnet18 and Cifar10. Bottom row: histogram comparisons of the gradient norms after 1 epoch of training (b) and 30 epochs of training (c).}%
 \label{fig:gn-over-epochs}%
\end{figure}

\begin{figure}[htbp!]
\centering
\begin{tabular}{cc}
\includegraphics[width=0.3\linewidth]{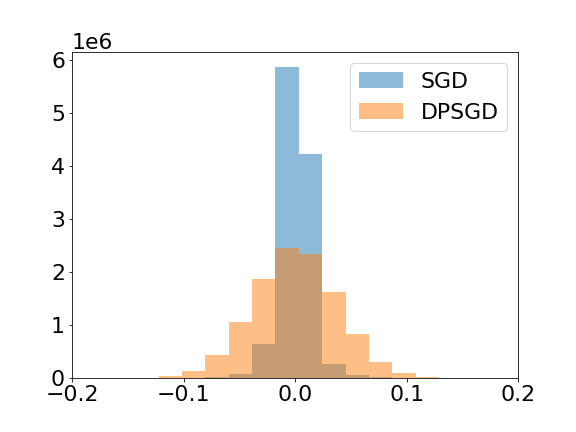} &
\includegraphics[width=0.3\linewidth]{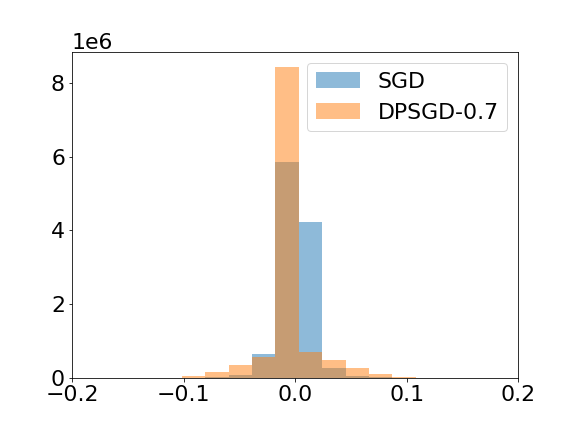}
\\
\end{tabular}
 \caption{DP-SGD models for $\epsilon=1$ have fewer parameters closer to 0 than SGD models, implying again that they occupy different regions of the parameter space. We find that SGD models have an L2 distance from the origin of $86\pm 0.5$ in comparison to $142\pm0.01$ for DP-SGD models. With pruning the DP-SGD distance reduces to $100\pm0.02$.  }
\label{fig:param-dists}
\end{figure}

\textbf{Other Models and Datasets:} In Figures~\ref{fig:mnist-eps1} and~\ref{fig:fashionmnist-eps1} we demonstrate that magnitude pruning increases the test accuracy of DP-SGD for a significantly smaller model (LeNet) using both the MNIST and FashionMNIST tasks.

\textbf{Varying the Phase 1 Training Duration:} In Figure~\ref{fig:pre_ep_comps}, we provide examples of different lengths of Phase 1 training ($k$) from Figure~\ref{fig:sgd-dpsgd-combos}. In all cases we see similar trends, with the DP-SGD model under-performing in comparison to SGD during the first few epochs, but also negatively impacting model performance later in training (Phase 2) when the model is initially trained with SGD.

\begin{figure}[htbp!]
\centering
\begin{tabular}{cc}
\includegraphics[width=0.3\linewidth]{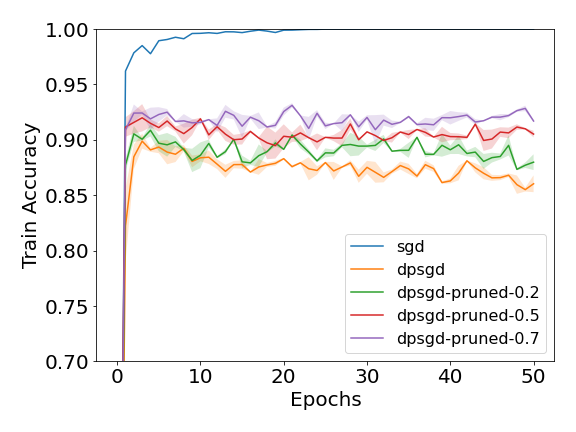} &
\includegraphics[width=0.3\linewidth]{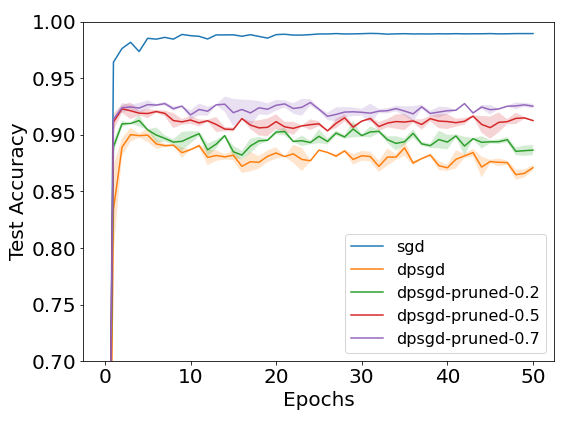}
\\
\end{tabular}
 \caption{Training and test accuracies for a CNN (LeNet) model with MNIST for $\epsilon=1$. Magnitude pruning performed via pre-training with MNIST.  }
\label{fig:mnist-eps1}
\end{figure}

\begin{figure}[htbp!]
\centering
\begin{tabular}{cc}
\includegraphics[width=0.3\linewidth]{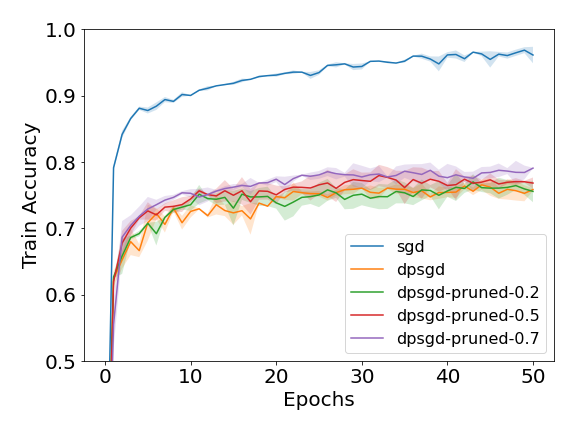} &
\includegraphics[width=0.3\linewidth]{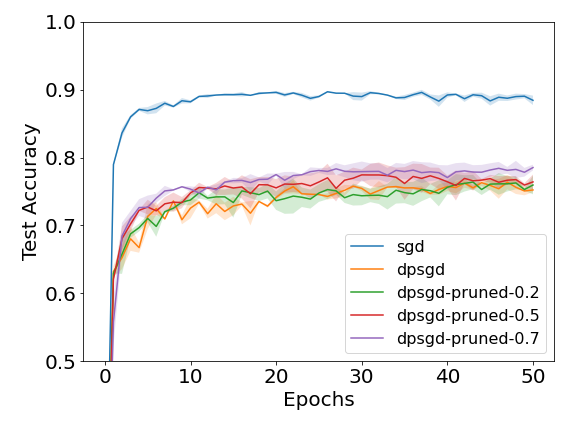}
\\
\end{tabular}
 \caption{Training and test accuracies for a CNN (LeNet) model with FashionMNIST for $\epsilon=1$. Magnitude pruning performed via pre-training with FashionMNIST.  }
\label{fig:fashionmnist-eps1}
\end{figure}

\begin{figure}[htbp!]
\centering
\begin{tabular}{cc}
\includegraphics[width=0.3\linewidth]{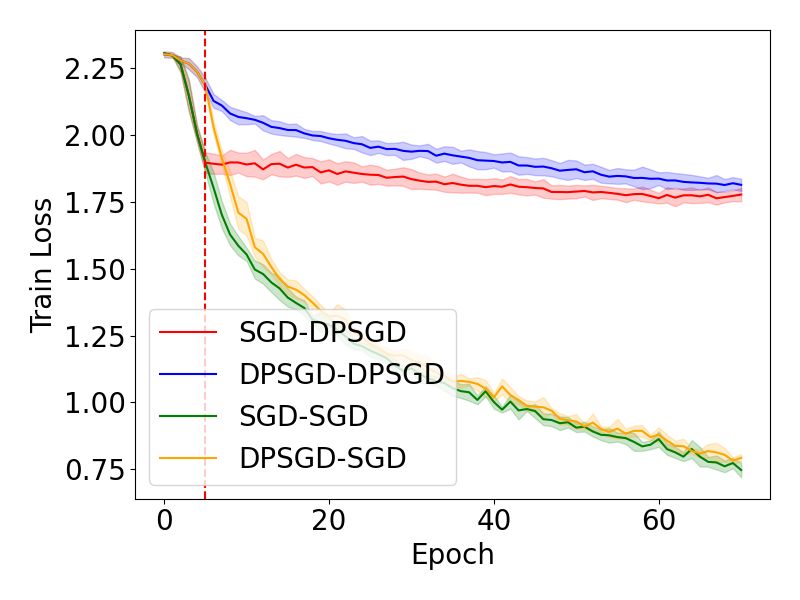} &
\includegraphics[width=0.3\linewidth]{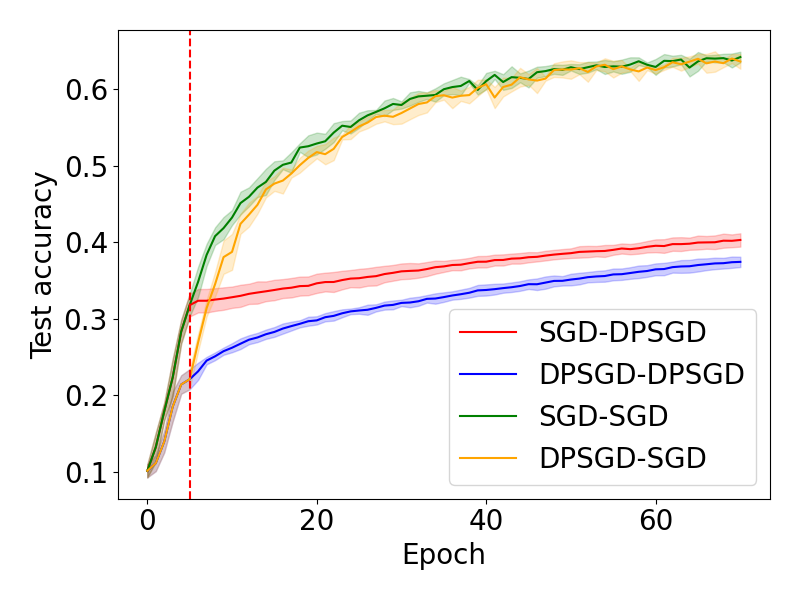}
\\
a) $k=5$ & b) $k=5$\\
\includegraphics[width=0.3\linewidth]{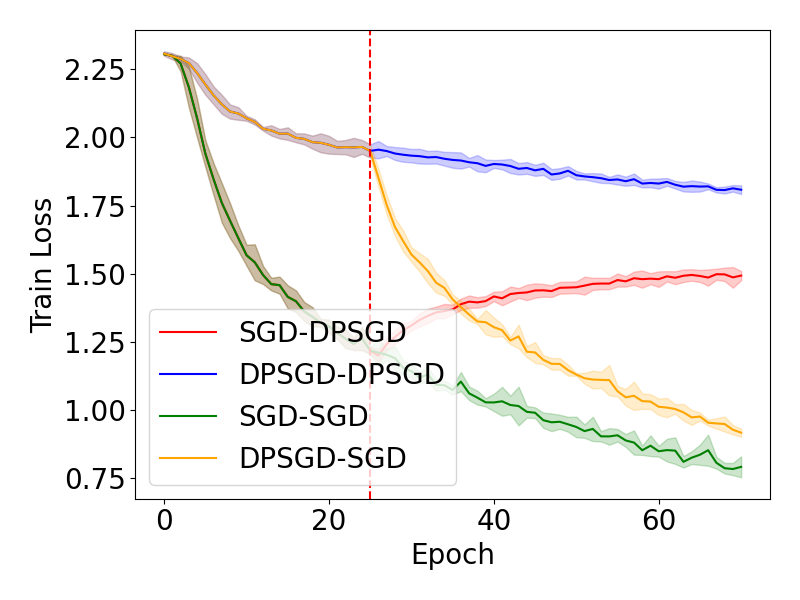} &
\includegraphics[width=0.3\linewidth]{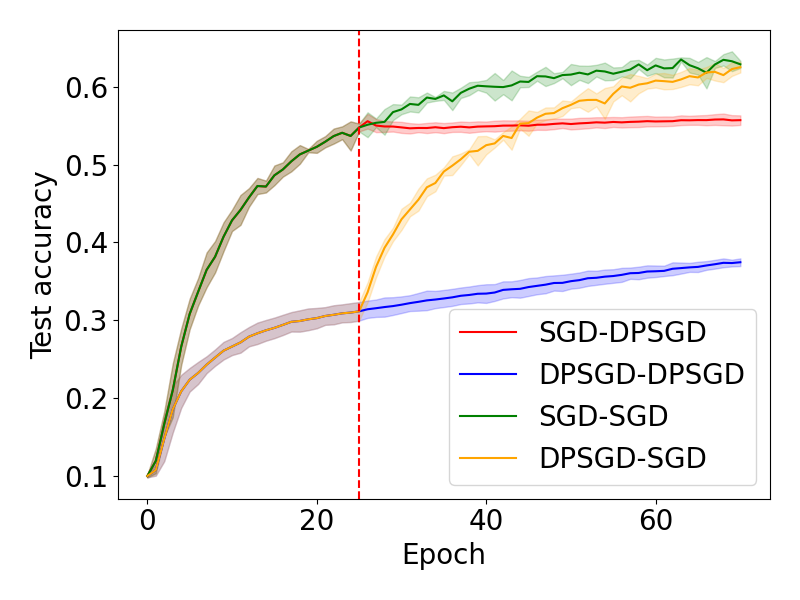}
\\
c) $k=25$ & d) $2k=5$\\

\includegraphics[width=0.3\linewidth]{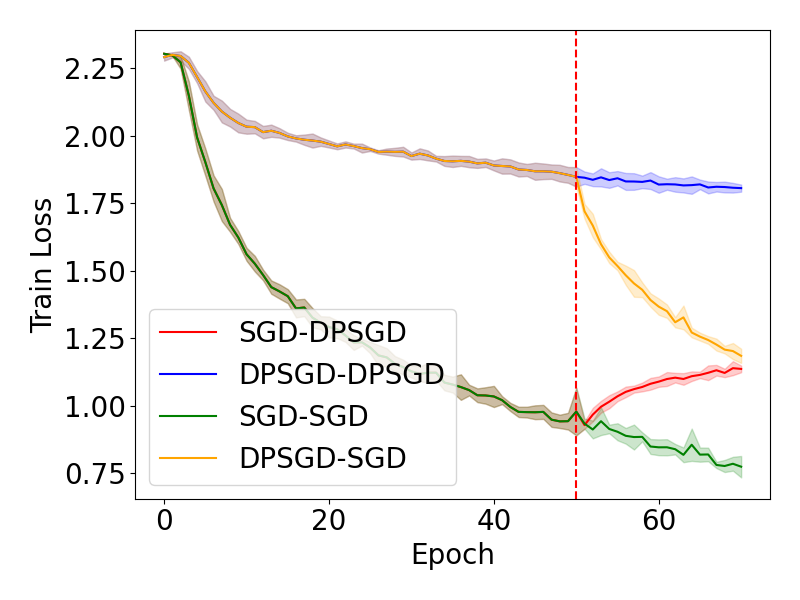} &
\includegraphics[width=0.3\linewidth]{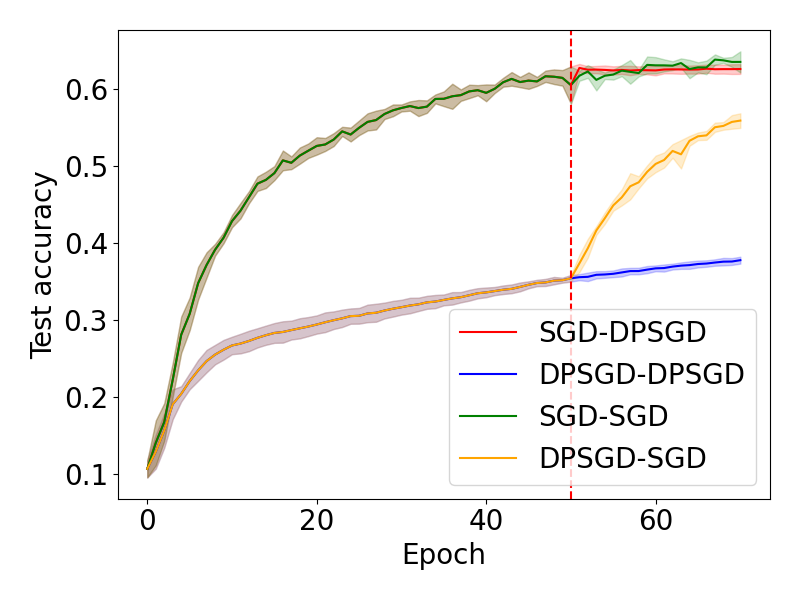}
\\
e)  $k=50$& f) $k=50$\\

\end{tabular}
 \caption{Performance with different Phase 1 and Phase 2 training methods. Later training epochs determine the final performance of both models for CIFAR10. DP-SGD used noise multiplier $\sigma=0.55$ and maximum gradient norm $C=1.0$ in each epoch, with $\epsilon\approx 7$ after training.   }
\label{fig:pre_ep_comps}
\end{figure}

\end{document}